\documentclass[pdflatex,sn-basic,Numbered]{sn-jnl}

\usepackage{graphicx}
\usepackage{multirow}
\usepackage{amsmath,amssymb,amsfonts}
\usepackage{booktabs}
\usepackage{tabularx}
\usepackage{url}
\usepackage{xcolor}
\definecolor{p3agent}{HTML}{4E79A7}    
\definecolor{p3cap}{HTML}{F28E2B}      
\definecolor{p3gov}{HTML}{59A14F}      
\definecolor{p3exec}{HTML}{9C9C9C}     
\definecolor{p3safe}{HTML}{76B7B2}     
\definecolor{p3alert}{HTML}{E15759}    
\definecolor{p3human}{HTML}{B07AA1}    
\definecolor{p3neutral}{HTML}{EDC948}  
\usepackage{tikz}
\usetikzlibrary{positioning, arrows.meta, fit, backgrounds, calc, shapes.geometric}
\usepackage{algorithm}
\usepackage{algpseudocode}
\usepackage{enumitem}
\setlist[itemize]{leftmargin=1.5em, itemsep=3pt, parsep=0pt}
\setlist[enumerate]{leftmargin=1.5em, itemsep=3pt, parsep=0pt}
\newcommand{\cmark}{\checkmark}
\newcommand{\pmark}{$\sim$}
\newcommand{\xmark}{--}

\definecolor{orcidlogocol}{HTML}{A6CE39}
\newcommand{\orcidicon}{%
  \mbox{\begin{tikzpicture}[x=1ex,y=1ex,baseline={([yshift=-0.45ex]current bounding box.center)}]%
    \fill[orcidlogocol] (0,0) circle (0.95);%
    \node[white,font=\fontsize{3pt}{3pt}\selectfont\bfseries\sffamily] at (0,0) {iD};%
  \end{tikzpicture}}%
}
\newcommand{\orcidlink}[1]{\href{https://orcid.org/#1}{\orcidicon}}

\newtheoremstyle{p3plain}%
  {6pt}{6pt}{\itshape}{0pt}{\bfseries}{.}{ }{}
\newtheoremstyle{p3definition}%
  {6pt}{6pt}{\normalfont}{0pt}{\bfseries}{.}{ }{}
\theoremstyle{p3definition}

\theoremstyle{p3plain}
\newtheorem{proposition}{Proposition}

\usepackage{placeins}

\begin{document}

\title[Harnessing Embodied Agents]{Harnessing Embodied Agents: Runtime Governance for Policy-Constrained Execution}

\author[1]{\fnm{Xue} \sur{Qin}~\orcidlink{0009-0009-3642-2663}}\email{qinxue@me.com}

\author[2]{\fnm{Simin} \sur{Luan}~\orcidlink{0000-0003-1138-1892}}\email{luansiminiot@gmail.com}

\author[3]{\fnm{John} \sur{See}~\orcidlink{0000-0003-3005-4109}}\email{J.See@hw.ac.uk}

\author[5]{\fnm{Zeyd} \sur{Boukhers}~\orcidlink{0000-0001-9778-9164}}\email{zeyd.boukhers@fit.fraunhofer.de}

\author*[4]{\fnm{Cong} \sur{Yang}~\orcidlink{0000-0002-8314-0935}}\email{cong.yang@suda.edu.cn}

\author*[2]{\fnm{Zhijun} \sur{Li}~\orcidlink{0000-0001-9129-9957}}\email{lizhijun\_os@hit.edu.cn}

\affil*[1]{\orgdiv{School of Software}, \orgname{Harbin Institute of Technology}, \orgaddress{\city{Harbin}, \country{China}}}

\affil[2]{\orgdiv{School of Computer Science and Technology}, \orgname{Harbin Institute of Technology}, \orgaddress{\city{Harbin}, \country{China}}}

\affil[3]{\orgdiv{School of Mathematical and Computer Sciences}, \orgname{Heriot-Watt University, Malaysia Campus}, \orgaddress{\city{Putrajaya}, \country{Malaysia}}}

\affil[4]{\orgdiv{School of Future Science and Engineering}, \orgname{Soochow University}, \orgaddress{\city{Suzhou}, \country{China}}}

\affil[5]{\orgname{Fraunhofer Institute for Applied Information Technology}, \orgaddress{\city{Sankt Augustin}, \country{Germany}}}

\abstract{Embodied Agents are evolving from passive reasoning systems into active executors that interact with tools, robots, and physical environments. Once an agent gains execution authority, the central challenge shifts from how to make it act to how to keep its actions governable at runtime. Existing approaches embed safety, recovery, and decision constraints inside the agent loop, making execution control difficult to standardize, audit, and adapt across environments. We propose a runtime governance framework for policy-constrained execution that separates agent cognition from execution oversight. Governance is externalized into a dedicated runtime layer performing policy checking, capability admission, execution monitoring, rollback, and human override. We formalize the control boundary among a persistent Embodied Agent, modular Capability Packages, and the governance layer, and define a policy-constrained execution pipeline evaluated under controlled simulation. Over 1000 randomized trials, the framework achieves $96.2\%\pm2.7\%$ interception of unauthorized actions, reduces unsafe continuation from 100\% to $22.2\%\pm3.1\%$ under runtime drift, and attains $90.7\%\pm3.0\%$ recovery success with full policy compliance. Comparison with five baselines, including AutoRT-style constitution filtering and RoboGuard-style two-stage guardrails, shows that pre-execution filtering is equally effective across governance-aware methods, while only the proposed framework provides continuous runtime detection (RVDR\,$= 61.3\%$ vs.\ 0\%) and structured recovery (all $p<0.001$). A sensitivity sweep across the full detection range confirms a genuine detection--continuation trade-off. This work argues future embodied systems should be designed for governable execution.}


\maketitle

\section{Introduction}
\label{sec:intro}

Embodied Agents are moving beyond passive reasoning and conversational assistance toward persistent execution in the physical world. Recent progress in large models~\citep{brohan2023rt1,brohan2023rt2,driess2023palme}, tool use~\citep{schick2023toolformer,yao2023react,vemprala2024chatgpt_robotics}, and robotic policy learning~\citep{ahn2022saycan,liang2023code} has made it increasingly plausible for a single agentic system to interpret goals, invoke capabilities, interact with robots and software tools, and complete long-horizon tasks across dynamic environments~\citep{huang2022inner,wang2023voyager}. In this emerging setting, however, the central challenge is no longer only how to make an agent act, but how to make its action governable at runtime. Figure~\ref{fig:teaser} contrasts ungoverned execution with our proposed governed execution, illustrating the role of the Runtime Governance Layer that mediates between agent cognition and physical execution.

This shift is fundamental. A system that can execute is not necessarily a system that can execute under explicit constraints, remain observable during operation, recover from failures in a controlled manner, or yield control when human intervention becomes necessary. As Embodied Agents gain access to tools, actuators, sensors, and real-world execution pathways, the cost of runtime failure becomes qualitatively different from that in purely digital settings. Errors are no longer limited to incorrect responses or failed API calls; they may instead result in unsafe motions, unauthorized actions, unstable recovery behaviors, or persistent deviations from deployment policy.

\begin{figure}[t]
\centering
\resizebox{\textwidth}{!}{%
\begin{tikzpicture}[
    >=Stealth,
    every node/.style={font=\footnotesize},
    agentbox/.style={draw=p3agent!60, rounded corners=3pt, fill=p3agent!10, minimum height=0.75cm, minimum width=1.7cm, align=center, line width=0.5pt},
    capbox/.style={draw=p3cap!60, rounded corners=3pt, fill=p3cap!10, minimum height=0.75cm, minimum width=1.7cm, align=center, line width=0.5pt},
    govbox/.style={draw=p3gov!60, rounded corners=3pt, fill=p3gov!10, minimum height=0.75cm, minimum width=1.7cm, align=center, line width=0.5pt},
    execbox/.style={draw=p3exec!60, rounded corners=3pt, fill=p3exec!8, minimum height=0.75cm, minimum width=1.7cm, align=center, line width=0.5pt},
    dangerlabel/.style={font=\scriptsize\itshape, text=p3alert!80!black},
    safelabel/.style={font=\scriptsize\itshape, text=p3gov!70!black},
    elabel/.style={font=\tiny\sffamily, text=black!55, inner sep=1pt},
    titlefont/.style={font=\footnotesize\bfseries\sffamily},
    dashbox/.style={draw, dashed, rounded corners=5pt, inner sep=8pt, line width=0.5pt},
]

\node[titlefont, text=p3alert!80!black] (titleL) at (-7.0, 3.5) {(a) Ungoverned Execution};

\node[agentbox] (agentL) at (-8.5, 1.6) {Embodied\\[-1pt] Agent};
\node[capbox]   (capL)   at (-5.5, 1.6) {Capability\\[-1pt] Package};
\node[execbox]  (execL)  at (-7.0, 0.0) {Physical\\[-1pt] Environment};

\draw[->, semithick] (agentL) -- (capL) node[midway, above, elabel] {invoke};
\draw[->, semithick] (capL) -- (execL) node[midway, right, elabel, xshift=2pt] {execute};
\draw[->, semithick, dashed, p3alert!50] (agentL) |- (execL) node[pos=0.75, left, elabel, xshift=-2pt] {direct};

\node[dangerlabel, text width=2.2cm, align=center] at (-7.0, -1.4) {No policy check\\No monitoring\\No recovery};

\draw[p3alert!40, line width=1.2pt] (-9.3, -2.1) -- (-4.7, -2.1);

\node[font=\normalsize\bfseries, text=black!45] at (-3.6, 0.7) {$\Longrightarrow$};
\node[font=\small\sffamily\bfseries, text=black!85, text width=1.8cm, align=center] at (-3.6, -0.25) {add runtime\\governance};

\node[titlefont, text=p3gov!80!black] (titleR) at (4.5, 3.5) {(b) Governed Execution (Ours)};

\node[agentbox] (agentR) at (-1.0, 1.6) {Embodied\\[-1pt] Agent};

\node[govbox] (admit)  at (2.0, 1.6) {Capability\\[-1pt] Admission};
\node[govbox] (policy) at (5.0, 1.6) {Policy\\[-1pt] Guard};
\node[govbox] (human)  at (8.0, 1.6) {Human\\[-1pt] Override};
\node[govbox] (watch)  at (3.5, 0.0) {Execution\\[-1pt] Watcher};
\node[govbox] (recov)  at (6.5, 0.0) {Recovery\\[-1pt] Manager};

\begin{scope}[on background layer]
\node[dashbox, fill=p3gov!4, draw=p3gov!50!black,
      fit=(admit)(policy)(watch)(recov)(human),
      label={[font=\tiny\sffamily\bfseries, text=p3gov!50!black]above:Runtime Governance Layer}] (govlayer) {};
\end{scope}

\node[capbox]  (capR)  at (3.5, -1.7) {Capability\\[-1pt] Package};
\node[execbox] (execR) at (6.5, -1.7) {Physical\\[-1pt] Environment};

\draw[->, semithick] (agentR) -- (admit) node[pos=0.35, above, elabel] {propose};
\draw[->, semithick] (admit) -- (policy) node[midway, above, elabel] {check};
\draw[->, semithick] (policy) -- (human) node[midway, above, elabel] {escalate};
\draw[->, semithick] (policy) -- (watch) node[midway, right, elabel, xshift=2pt] {launch};
\draw[->, semithick] (watch) -- (recov) node[midway, above, elabel] {trigger};
\draw[->, semithick, p3gov!60!black] (watch) -- (capR) node[pos=0.65, left, elabel, xshift=-2pt] {govern};
\draw[->, semithick] (capR) -- (execR) node[midway, above, elabel] {execute};

\node[safelabel, text width=3.4cm, align=center] at (5.0, -2.5) {Policy-checked $\cdot$ Monitored $\cdot$ Recoverable};

\draw[->, semithick, dashed, p3agent!40] (recov.south) -- (execR.north) node[pos=0.5, right, font=\tiny\sffamily, text=black!50] {audit log};

\end{tikzpicture}%
}
\caption{\textbf{Core idea.} (a)~In ungoverned execution, the Embodied Agent directly invokes capabilities without policy checking, runtime monitoring, or recovery. (b)~Our framework interposes a \emph{Runtime Governance Layer} between agent cognition and physical execution: every capability invocation passes through admission control, policy checking, and execution monitoring, with recovery and human override at runtime. This separation makes execution policy-constrained, observable, recoverable, and auditable without modifying the agent model.}
\label{fig:teaser}
\end{figure}
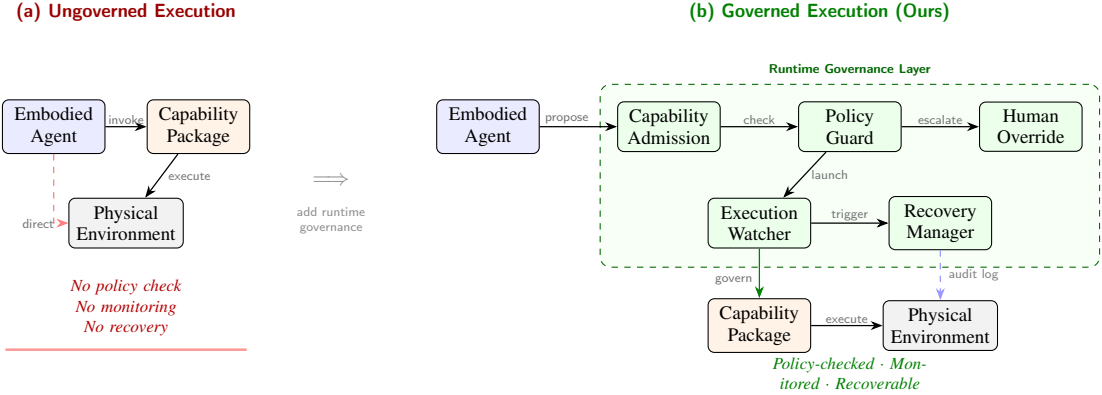

Existing Embodied Agent systems have made important progress in planning, action generation, policy learning, and task execution~\citep{huang2023voxposer,ahn2022saycan,embodied_survey}. Yet in many cases, the logic that governs execution remains entangled with the agent itself. Safety heuristics, recovery strategies, approval conditions, and action constraints are often embedded inside prompts, agent loops, task policies, or ad hoc controller logic. While such designs may be effective for narrow tasks or tightly controlled demonstrations, they become increasingly difficult to standardize, audit, verify, and adapt when the same Embodied Agent must operate across simulation, real robots, changing deployment contexts, and human-facing environments.

This paper argues that embodied intelligence requires both stronger agents and stronger runtime governance. We refer to this perspective as \textbf{harnessing Embodied Agents}: instead of assuming that the agent itself should internalize all safety, recovery, and execution control logic, we propose that these responsibilities should be externalized into a dedicated Runtime Governance Layer. The agent remains responsible for task understanding, planning, and capability invocation; the governance layer determines whether, when, and how execution may proceed under explicit policy constraints. Our focus is therefore on the \emph{execution boundary}, the operational layer at which an Embodied Agent transitions from intention to action, and where capability admission, policy enforcement, execution monitoring, rollback handling, and human override must be made explicit.

Based on this perspective, we propose a runtime governance framework for \textbf{policy-constrained execution}. The governance layer operates on an explicit, inspectable policy knowledge base comprising admission rules, constraint predicates, environment profiles, and recovery strategies, enabling knowledge-driven runtime decision-making over embodied execution. Structurally, the layer instantiates a Monitor--Analyze--Plan--Execute (MAPE-K) feedback loop~\citep{kephart2003autonomic} over a policy knowledge base: the Execution Watcher monitors runtime telemetry, the Policy Guard analyzes compliance, the Recovery Manager plans corrective action, and the Capability Package executes under governance authority. The framework separates three roles often conflated in current embodied systems: the persistent Embodied Agent as the decision-making subject, the Capability Package as the executable unit, and the Runtime Governance Layer as the authority that constrains and supervises execution. This separation builds on the single-agent runtime architecture and Embodied Capability Module (ECM) abstractions introduced in prior work~\citep{qin2026aeros}. This separation is motivated by three observations: (i)~embodied systems increasingly require persistent agents that operate across tasks rather than isolated single-behavior policies; (ii)~the executable abilities available to such agents are becoming modular and updatable, making capability-level governance more practical than monolithic control; and (iii)~deployment environments differ substantially in their acceptable operational boundaries: actions permissible in simulation may be unsafe on a real robot or require approval in a human-shared setting.

We formalize the control boundary among these three entities and define a policy-constrained execution pipeline in which agent-proposed actions are mediated through capability admission, policy checking, execution oversight, anomaly-triggered interruption, rollback handling, and human override. Rather than treating governance as auxiliary safeguards, the proposed design treats it as a core runtime structure, making it possible to reason about governability independently of the agent model and to adapt operational policy without rewriting the agent itself.

It is important to clarify the scope of this work. This paper is not primarily an end-to-end robot controller paper, nor a model capability paper. It is a \emph{runtime governance framework paper} for embodied execution. Our goal is not to compete with agentic robot controllers~\citep{brohan2023rt2,ahn2022saycan} on task success metrics, but to address the complementary question of how embodied execution can remain policy-constrained, observable, recoverable, and auditable as agent capability increases.

The contributions of this paper are as follows:

\begin{enumerate}[leftmargin=*]
    \item We identify \textbf{runtime governance} as a distinct systems problem in embodied AI and argue that policy-constrained execution should be treated as a foundational design principle for Embodied Agent systems.

    \item We propose a framework for \textbf{harnessing Embodied Agents}, in which agent cognition is separated from execution oversight through an explicit Runtime Governance Layer.

    \item We formalize the control boundary among a persistent Embodied Agent, modular Capability Packages, and runtime governance mechanisms, and define a unified execution pipeline designed for simulation with a physical-robot deployment~path.

    \item We present an evaluation protocol covering unauthorized action interception, high-risk action gating, runtime recovery, and cross-environment policy adaptation, enabling systematic study of governable execution in embodied~settings.
\end{enumerate}

\textbf{What is new.} Prior embodied-AI systems incorporate individual safety mechanisms (collision avoidance, action filtering, human confirmation), and the autonomic-computing literature provides the MAPE-K reference model. Our contribution is neither a new safety mechanism nor a new feedback-loop architecture in isolation. What is new is the \emph{governance separation principle}: the systematic externalization of all runtime governance responsibilities (admission, policy enforcement, monitoring, recovery, and human override) into a dedicated knowledge-based layer that is (i)~formally independent of the agent model, (ii)~parameterized by an inspectable policy knowledge base rather than embedded in agent logic, and (iii)~provably complete in its mediation of the execution boundary (Proposition~\ref{prop:regimentation}) with quantifiable detection guarantees (Proposition~\ref{prop:bounded}). No prior work in embodied AI provides this combination: existing approaches either embed governance within the agent (making it unauditable and policy-coupled) or provide point solutions for a single governance function (monitoring without admission, or filtering without recovery). The separation principle yields a qualitatively different engineering property: the governance layer can be formally analyzed, independently tested, and reconfigured across deployment contexts without modifying the agent, its planner, or its learned policies. This is the core novelty claim of the paper.

\textbf{Terminology note.} Throughout this paper, we use the term \emph{runtime governance framework} (or simply \emph{governance framework}) to refer to the overall proposed system design, while \emph{Runtime Governance Layer} refers specifically to Stratum~3 in the architecture (Section~\ref{sec:architecture}), the concrete middleware component that mediates execution. The word ``harness'' in the title refers to the act of harnessing (constraining and directing) Embodied Agents, not to a separate technical concept. When discussing external work on ``harness engineering''~\citep{harness_engineering}, we retain their original terminology.

The remainder of this paper is organized as follows. Section~\ref{sec:related} reviews related work. Section~\ref{sec:method} presents our method in four parts: the problem and system model (Section~\ref{sec:formulation}), the runtime governance framework (Section~\ref{sec:architecture}), the policy-constrained execution pipeline (Section~\ref{sec:pipeline}), and a prototype implementation (Section~\ref{sec:prototype}). Section~\ref{sec:experiments} presents the evaluation protocol. Section~\ref{sec:discussion} discusses implications and limitations. Section~\ref{sec:conclusion} concludes.

\section{Related Work}
\label{sec:related}

\subsection{Agentic Systems and the Rise of Runtime Harnesses}

Recent agentic systems have shifted attention from passive language interaction toward systems that can invoke tools, operate software environments, and persist across multi-step tasks~\citep{wang2024llmagent_survey,lu2023chameleon}. In this context, the surrounding execution environment has become increasingly important. Recent discussions around \emph{harness engineering} frame this shift as a move away from optimizing the agent in isolation and toward designing the tooling, guardrails, evaluations, and runtime control surfaces that keep agents ``in check''~\citep{harness_engineering}. This emerging view is particularly relevant to our work because it highlights that the practical bottleneck is no longer only agent intelligence, but the systems layer that mediates execution.

At the same time, recent work around OpenClaw-like ecosystems~\citep{openclaw} further illustrates both the promise and risk of execution-capable agents. For example, recent studies on OpenClaw analyze how personalized local agents can be manipulated or subverted~\citep{openclaw}, emphasizing that powerful execution authority creates new attack surfaces and governance problems beyond ordinary chatbot settings. These developments reinforce the broader argument of this paper: once agents can act, runtime governance becomes a core systems concern rather than an optional safety add-on.

However, most current discussions of harnesses remain centered on software agents, coding agents, or desktop execution environments. Our work differs in that it studies runtime governance in \textbf{embodied} settings, where execution unfolds over time in physical environments, affects actuators and sensors, and may require interruption, rollback, and human takeover under deployment-specific policy constraints. In this sense, we extend the emerging harness perspective into embodied AI and robot software systems.

\subsection{Embodied Agents and Long-Horizon Robotic Execution}

Embodied AI has recently moved toward increasingly capable agentic control of robot behavior, especially in language-conditioned and long-horizon settings. SayCan~\citep{ahn2022saycan} grounds language commands in robotic affordances to select feasible actions, Code as Policies~\citep{liang2023code} generates robot control code directly from language instructions, and RT-2~\citep{brohan2023rt2} trains vision-language-action models that unify perception and control into a single policy. These systems demonstrate that embodied agents can increasingly integrate reasoning, policy selection, and execution into unified control loops for long-horizon robotic tasks.

This line of work is important and complementary to our agenda rather than opposed to it. The systems above primarily address how an agentic controller can achieve scalable task execution. By contrast, we focus on a different systems question: once an Embodied Agent has execution authority, how should its actions be constrained, monitored, interrupted, recovered, and supervised at runtime? In other words, prior work has made strong progress on \textbf{how Embodied Agents act}, while our work focuses on \textbf{how embodied execution remains governable while acting}.

The broader Embodied Agent landscape includes systems that demonstrate persistent, adaptive agent behavior: Generative Agents~\citep{park2023generative} show that LLM-backed agents can maintain long-term behavioral coherence, ProgPrompt~\citep{singh2023progprompt} introduces programmatic prompting for situated robot task planning, and Voyager~\citep{wang2023voyager} demonstrates open-ended lifelong learning in embodied settings. Recent surveys~\citep{embodied_survey,liu2024embodied_survey} repeatedly identify safety, real-time enforcement, and runtime reliability as unresolved issues in practical deployment. Simulation platforms such as CARLA~\citep{carla2017} and Gazebo~\citep{gazebo} have enabled safety benchmarking but focus on environment fidelity rather than governance architecture. A recent benchmark, EARBench~\citep{zhu2024earbench}, directly evaluates physical risk awareness in foundation-model-based Embodied Agents, and SafeAgentBench~\citep{yin2024safeagentbench} provides a comprehensive benchmark for safe task planning of embodied LLM agents, reporting that even the most safety-conscious baseline achieves only a 10\% rejection rate for hazardous tasks. These results provide empirical evidence that current systems lack robust risk recognition, a gap our governance layer is designed to address. Fault detection, isolation, and recovery (FDIR) techniques from space robotics~\citep{tipaldi2015fdir} provide precedent for structured recovery pipelines, yet remain domain-specific. These surveys note the need for lightweight runtime monitoring, bounded-overhead safety enforcement, and robust execution under real-world uncertainty. Our work contributes to this gap by making runtime governance an explicit architectural and evaluation target rather than treating it as a secondary implementation detail.

\subsection{Safe Robotics, Runtime Monitoring, and Constraint Enforcement}

A large body of robotics research has addressed safety through controller design, reactive shielding~\citep{alshiekh2018shielding,konighofer2020shield}, control barrier functions~\citep{ames2019cbf,fisac2019general}, constrained policy optimization~\citep{achiam2017cpo,dalal2018safe}, and runtime monitoring~\citep{desai2017runtime}. Comprehensive surveys of safe reinforcement learning~\citep{garcia2015saferl,brunke2022safe} and the broader landscape of robust control for autonomous systems~\citep{recht2019tour} establish that embodied execution cannot be treated as unconstrained action generation. The Simplex architecture~\citep{sha2001simplex} pioneered the idea of a high-assurance controller that can override a high-performance but less verified one, prefiguring our notion of externalized runtime authority. Similarly, the Seldonian framework~\citep{thomas2019seldonio} demonstrated that safety constraints can be enforced during policy improvement without embedding them inside the learning objective itself. Industrial and service-robot safety standards~\citep{iso10218,iso13482} further codify the requirement for runtime monitoring, emergency stop, and human override in deployed systems. Systems-theoretic safety engineering~\citep{leveson2011engineering} and autonomous-vehicle safety analysis~\citep{koopman2019safety} reinforce the broader principle that safety is a control problem, not merely a component property.

Recent work on LLM-based anomaly detection for robotics~\citep{sinha2024anomaly} demonstrates that foundation models can perform real-time anomaly detection and reactive replanning, providing a complementary capability-level mechanism to our governance-layer watcher design.

Yet much of this literature focuses either on low-level safety enforcement or on task-specific pipelines, rather than on the systems question of how a persistent Embodied Agent should be governed across modular capabilities, shifting environment profiles, and human-supervised execution modes. Our work builds on the intuition behind runtime safety and monitoring, but lifts the perspective from controller-level safety to \textbf{runtime governance of policy-constrained execution}, covering the full lifecycle from capability proposal to admission, policy enforcement, watcher-driven intervention, rollback, and audit.

\subsection{Autonomic Computing, Self-Adaptive Systems, and Runtime Governance}

The idea of externalizing execution control into an enforceable middleware layer draws on at least three research traditions. First, the \emph{autonomic computing} vision~\citep{kephart2003autonomic} introduced the MAPE-K (Monitor--Analyze--Plan--Execute over a shared Knowledge base) reference model for self-managing systems; the Runtime Governance Layer in this paper follows the same closed-loop structure, with telemetry flowing from the Execution Watcher (Monitor) through the Policy Guard (Analyze) and Recovery Manager (Plan/Execute) over a policy knowledge base. Second, the broader \emph{self-adaptive systems} literature~\citep{weyns2019selfadaptive} studies how software can reconfigure its own behaviour at runtime in response to changing goals or environmental conditions, a concern directly mirrored by our environment-profile-parameterized governance. Third, \emph{formal verification of autonomous robotic systems}~\citep{luckcuck2019fv_robotic} has catalogued the gap between verified low-level controllers and the largely unverified supervisory layers that coordinate them, motivating the kind of explicit governance architecture proposed here.

Within multi-agent systems, the governance-layer concept has equally deep roots. The BDI (Belief--Desire--Intention) architecture~\citep{rao1995bdi} established the agent-theoretic foundation on which runtime governance operates: a persistent agent with explicit mental attitudes whose intentions can be intercepted and constrained before they become actions. Policy-driven management~\citep{sloman1994policy} showed early on that distributed systems benefit from declarative, externalized policy specifications rather than hard-coded control logic, a principle our policy knowledge base directly inherits. Electronic Institutions~\citep{esteva2001electronic} introduced the notion of a computational institution that constrains agent interactions through predefined dialogic frameworks and role-based protocols, establishing an early form of the ``governance layer'' concept that this paper instantiates for embodied execution. Law-Governed Interaction~\citep{minsky2000lgi} formalized the principle that interaction policies should be enforced by an external mechanism rather than left to each agent's internal compliance, a principle directly reflected in our design choice to externalize governance from agent cognition. Within the normative MAS tradition, Grossi et al.~\citep{grossi2007norm} drew a foundational distinction between \emph{regimentation} (blocking non-compliant actions before they occur) and \emph{sanctioning} (penalizing violations after the fact); our Admission Gate and Policy Checker implement regimentation, while the Recovery Manager handles post-violation response. Modgil and Luck~\citep{modgil2009monitoring} developed monitoring architectures for norm compliance in agent-based systems, providing a precursor to our Execution Watcher. Falcone et al.~\citep{falcone2011enforcement} grounded runtime enforcement monitors in a formal allow/modify/deny semantics that directly informs the four-valued output of our Policy Check function. The MAOP paradigm and its JaCaMo realization~\citep{boissier2013jacamo} demonstrated how organisational constraints can be separated from agent programs and enforced through dedicated environmental artifacts, an architectural pattern closely aligned with our three-entity control boundary. From a knowledge-based systems perspective, the governance layer can be understood as a runtime reasoning engine that evaluates structured policy knowledge (admission predicates, constraint rules, environment profiles) against the current execution state to produce governance decisions. This framing connects the normative-MAS tradition to the knowledge-based control and rule-based reasoning literature familiar to the intelligent-systems community.

What changes when this normative MAS tradition meets physically grounded execution is the nature of the state space and the consequences of enforcement failure. In software MAS, a blocked message or denied action has no physical residue; in embodied settings, continuous physical state (joint positions, contact forces, spatial proximity to humans) means that governance must reason over real-time telemetry rather than discrete message passing. Action irreversibility (a dropped object, a collision, a force exceedance) demands structured rollback and recovery mechanisms that have no direct analogue in classical norm enforcement. Environment profiles (Section~\ref{sec:formulation}) parameterize governance rules by deployment context, extending the norm-parameterization ideas of Morales et al.~\citep{morales2015norms}, who synthesized liberal normative systems that minimize unnecessary restriction, into the physical domain where ``unnecessary restriction'' may mean a robot freezing mid-task. Human override, treated in the adjustable-autonomy literature~\citep{scerri2002adjustable} as a transfer-of-control problem, becomes a structural pipeline stage because embodied tasks often cannot safely pause without explicit state preservation. Shield synthesis~\citep{bloem2015shield} and norm approximation~\citep{alechina2014norm} offer formal tools for constructing correct-by-construction enforcement and for reasoning about norm compliance under partial observability, respectively; both are relevant to future extensions of the framework presented here. In summary, our work instantiates the governance agenda of normative MAS into the embodied execution domain, retaining the core architectural insight (externalized, enforceable constraints) while addressing the physical grounding, real-time monitoring, and irreversibility challenges that embodied deployment introduces.

\subsection{Runtime Enforcement for LLM-Based Agents}

Recent work on safe LLM agents has started to explore \emph{runtime enforcement} more explicitly. \textbf{AutoRT}~\citep{ahn2024autort} demonstrates large-scale orchestration of robotic agents using a ``robot constitution'' that filters unsafe task proposals through an LLM critic. This is the closest prior system to our governance approach, though AutoRT's constitution operates as a pre-execution filter rather than a continuous Runtime Governance Layer with monitoring, recovery, and human override. \textbf{SafeEmbodAI}~\citep{zhang2024safeembodai} proposes a safety framework for mobile robots in embodied AI systems, incorporating secure prompting and state management to mitigate malicious command injection, but does not address runtime execution monitoring or structured recovery. \textbf{GuardAgent}~\citep{xiang2024guardagent} introduces a guard agent that dynamically generates guardrail code to protect target agents, achieving over 98\% accuracy on healthcare access control benchmarks, though it focuses on digital agent safety rather than embodied execution. \textbf{RoboGuard}~\citep{ravichandran2025roboguard} proposes a two-stage guardrail architecture for LLM-enabled robots that reduces unsafe plan execution from 92\% to below 2.5\%, which makes it the most directly comparable embodied safety system; however, RoboGuard focuses on pre-execution plan filtering rather than continuous runtime governance with recovery and rollback.

\textbf{AgentSpec}~\citep{agentspec} studies customizable runtime enforcement for LLM agents and evaluates enforcement on multiple domains, including an Embodied Agent scenario involving a robotic arm. NeMo Guardrails~\citep{rebedea2023nemo} provides a toolkit for programmable runtime rails on LLM applications, and recent work proposes safety guardrails specifically for LLM-enabled robots~\citep{ravichandran2025roboguard}. \textbf{TrustAgent}~\citep{hua2024trustagent} introduces an ``agent constitution'' with pre-planning, in-planning, and post-planning safety strategies, a pipeline structure that parallels our admission/monitoring/recovery stages, though applied to digital rather than Embodied Agents. \textbf{Pro2Guard}~\citep{wang2025pro2guard}, by the AgentSpec group, extends runtime enforcement to probabilistic violation prediction using Markov chain models, providing \emph{proactive} intervention before violations occur, a complementary capability to our reactive Execution Watcher. Runtime verification theory~\citep{leucker2009rv} provides the formal underpinnings for monitoring execution traces against temporal properties, a perspective that informs our Execution Watcher design.

Our work differs from all of the above in scope: we formulate governance around three explicit entities (persistent Embodied Agent, modular Capability Packages, and Runtime Governance Layer) and emphasize \emph{continuous} runtime governance, encompassing pre-execution filtering together with execution watching, recovery and rollback as governance functions, environment-sensitive policy profiles, and human override as a structural pipeline component. In that sense, our contribution is less about general safe-agent enforcement and more about defining a systems architecture for \textbf{governable embodied execution}.

\subsection{Position of This Work}

Prior work suggests a convergence: agentic systems are increasingly judged by their runtime structures, embodied systems are rapidly increasing the scope of persistent robotic execution, and safety enforcement is becoming more central as deployment moves beyond demonstrations. This paper sits at the intersection of these trends, contributing a dedicated runtime governance perspective for embodied AI.

To clarify how our contribution relates to prior work, we note the following decomposition. Several individual governance mechanisms in our framework have roots in autonomic computing, self-adaptive systems, normative MAS, and established engineering patterns: the overall closed-loop structure follows the MAPE-K reference model~\citep{kephart2003autonomic}; environment-profile-parameterized adaptation mirrors the self-adaptive systems paradigm~\citep{weyns2019selfadaptive}; capability admission draws on access-control models from operating systems and middleware, and on the regimentation concept from normative MAS~\citep{grossi2007norm}; policy-based gating extends constraint-checking mechanisms found in Simplex~\citep{sha2001simplex}, AgentSpec~\citep{agentspec}, AutoRT's robot constitution~\citep{ahn2024autort}, policy-driven management~\citep{sloman1994policy}, and the law-governed interaction paradigm~\citep{minsky2000lgi}; execution watching builds on runtime verification~\citep{leucker2009rv,desai2017runtime}, formal verification of robotic systems~\citep{luckcuck2019fv_robotic}, anomaly detection~\citep{sinha2024anomaly}, and norm monitoring architectures~\citep{modgil2009monitoring}; recovery/rollback draws on FDIR techniques from space robotics~\citep{tipaldi2015fdir}; and the externalized governance layer itself follows the organisational separation advocated by Electronic Institutions~\citep{esteva2001electronic} and MAOP/JaCaMo~\citep{boissier2013jacamo}. What is \emph{new} in this paper is: (i)~composing these mechanisms from autonomic computing, normative MAS, and runtime verification into a unified, continuous Runtime Governance Layer designed for physically grounded Embodied Agents; (ii)~the three-entity control boundary (agent / capability / governance) that separates cognition from execution oversight; (iii)~environment-profile-parameterized governance that adapts the same framework across deployment contexts without modifying the agent; and (iv)~the treatment of governance as a lifecycle process (not a one-time pre-check) with recovery, rollback, and human override as first-class pipeline stages. While the individual mechanisms echo established patterns across these traditions, no prior system combines all of them into a continuous governance lifecycle for physically grounded embodied~agents.

This paper addresses the single autonomous agent case; multi-agent extensions with shared-resource arbitration are a natural next step that builds on the per-agent governance primitives defined here.

To further clarify our positioning, Table~\ref{tab:comparison} compares the proposed framework with six representative approaches along key governance dimensions.

\begin{table}[t]
\centering
\caption{Qualitative comparison of runtime governance approaches along key governance dimensions. \cmark\ = supported, \pmark\ = partial, \xmark\ = not supported. Column abbreviations: Splx = Simplex; ASpec = AgentSpec; NeMo = NeMo Guardrails; AuRT = AutoRT; RoGd = RoboGuard; EI/N.\ = EI/Normative MAS~\citep{esteva2001electronic,grossi2007norm,boissier2013jacamo}.}
\label{tab:comparison}
\footnotesize
\begin{tabular*}{\textwidth}{@{\extracolsep{\fill}}lccccccc@{}}
\toprule
\textbf{Dimension} & \textbf{Splx} & \textbf{ASpec} & \textbf{NeMo} & \textbf{AuRT} & \textbf{RoGd} & \textbf{EI/N.} & \textbf{Ours} \\
\midrule
Capability admission     & \xmark & \pmark & \xmark & \cmark & \pmark & \cmark & \cmark \\
Policy-based gating      & \pmark & \cmark & \cmark & \cmark & \cmark & \cmark & \cmark \\
Runtime execution watch  & \cmark & \pmark & \xmark & \xmark & \xmark & \pmark & \cmark \\
Recovery \& rollback     & \cmark & \xmark & \xmark & \xmark & \xmark & \xmark & \cmark \\
Human override interface & \xmark & \xmark & \xmark & \xmark & \xmark & \pmark & \cmark \\
Audit \& telemetry       & \xmark & \pmark & \pmark & \pmark & \xmark & \pmark & \cmark \\
Environment profiles     & \xmark & \xmark & \xmark & \xmark & \xmark & \xmark & \cmark \\
Embodied-specific design & \pmark & \pmark & \xmark & \cmark & \cmark & \xmark & \cmark \\
\bottomrule
\end{tabular*}
\end{table}

\section{Method}
\label{sec:method}

This section presents our approach to runtime governance for policy-constrained execution of Embodied Agents. We first formalize the problem and define the system model (Section~\ref{sec:formulation}), then describe the runtime governance framework architecture (Section~\ref{sec:architecture}), detail the seven-stage policy-constrained execution pipeline (Section~\ref{sec:pipeline}), and outline the prototype implementation (Section~\ref{sec:prototype}).

\subsection{Problem Formulation and System Model}
\label{sec:formulation}

\subsubsection{Problem Statement}

We consider an Embodied Agent system that operates over extended time horizons and interacts with physical environments through executable capabilities. Unlike a purely conversational agent, such a system does not terminate at response generation. Instead, it must transform high-level intent into situated execution, often by invoking skills, tools, controllers, or robot-specific procedures under changing environmental and operational conditions.

The central problem addressed in this paper is the following:

\begin{quote}
\emph{How can an Embodied Agent be allowed to execute persistently and adaptively, while supporting governable execution under explicit runtime constraints?}
\end{quote}

This problem arises because embodied execution introduces a structural tension between \textbf{agent autonomy} and \textbf{operational control}. On the one hand, the agent must retain enough autonomy to interpret goals, compose actions, and respond to changing context. On the other hand, deployment systems must ensure that such actions remain bounded by safety requirements, environment-specific rules, organizational policy, and human supervisory authority. When these constraints are handled implicitly inside the agent loop, execution governance becomes difficult to inspect, standardize, or transfer across settings.

We therefore formulate \textbf{policy-constrained execution} as a systems problem rather than a purely model-level problem. In our view, the main challenge extends beyond whether an Embodied Agent can produce an action plan to whether the resulting execution can be admitted, monitored, interrupted, recovered, and audited by an explicit runtime mechanism.

\subsubsection{System Assumptions}

We assume an embodied system with the following properties:

\begin{enumerate}[leftmargin=*]
    \item \textbf{Persistent Agent Identity.} The system contains a persistent Embodied Agent that maintains task continuity across interactions rather than behaving as a stateless task-specific controller.
    \item \textbf{Modular Executable Capabilities.} The agent does not directly implement all low-level execution logic itself. Instead, it invokes modular executable units, referred to here as Capability Packages, which encapsulate skills, controllers, tools, or executable procedures.
    \item \textbf{Environment-Dependent Constraints.} The same capability may be permissible under one deployment condition but restricted under another. For example, actions that are acceptable in simulation may be disallowed or require approval on a physical robot.
    \item \textbf{Runtime Variability and Failure.} Execution may fail due to sensor noise, tool unavailability, actuation instability, policy mismatch, or environmental disturbance. The system must therefore support runtime monitoring and recovery.
    \item \textbf{Human Supervisory Possibility.} In at least some settings, humans may review, interrupt, approve, or override execution. Human intervention is not treated as an afterthought, but as part of the governance model.
\end{enumerate}

These assumptions are intentionally broad. They apply to systems spanning simulator-based Embodied Agents, robot manipulation agents, mobile robot assistants, and mixed cyber-physical agent platforms.

\subsubsection{Core Entities}

We define three primary entities in the system: the \textbf{Embodied Agent}, the \textbf{Capability Package}, and the \textbf{Runtime Governance Layer}.

\paragraph{Embodied Agent.}

The Embodied Agent is the persistent decision-making subject of the system. It is responsible for interpreting user goals or task objectives, maintaining task-level context and continuity, selecting or composing capabilities, proposing execution plans, and reacting to runtime feedback at the planning level.\looseness=-1

The Embodied Agent is \emph{not} assumed to have unrestricted execution authority. Its role is to generate intended actions and capability invocation requests, but not to unilaterally determine whether these actions are allowed to proceed in the current operational context.

Formally, let the Embodied Agent at time step $t$ be represented as:
\begin{equation}
    A_t = (\mathcal{I}_t, \mathcal{M}_t, \mathcal{G}_t, \mathcal{P}_t)
\end{equation}
where $\mathcal{I}_t$ denotes the agent's current identity and state continuity, $\mathcal{M}_t$ denotes memory and task-relevant context, $\mathcal{G}_t$ denotes active goals, and $\mathcal{P}_t$ denotes the proposed plan or action intention~\citep{rao1995bdi}. The key point is that $\mathcal{P}_t$ is a proposal for execution, not execution itself.

\paragraph{Capability Package.}

A Capability Package is an executable unit that encapsulates a bounded operational function. Depending on the deployment setting, it may contain a robot skill, a motion primitive, a controller wrapper, a tool-use procedure, a perception-action routine, a recovery behavior, or a composite workflow.

A Capability Package mediates between high-level agent intent and concrete execution substrate. Each package is assumed to expose a machine-readable interface and metadata sufficient for runtime governance. We represent a Capability Package $C_i$ as:
\begin{equation}
    C_i = (\textsf{\scriptsize name, interface, preconditions, postconditions, permissions, risk, rollback, env-profile})
\end{equation}
where \textbf{name} identifies the capability, \textbf{interface} specifies inputs, outputs, and invocation structure, \textbf{preconditions} define the conditions under which execution is valid, \textbf{postconditions} describe expected outcomes or completion signals, \textbf{permissions} specify required authority or policy scope, \textbf{risk} characterizes operational risk level, \textbf{rollback} specifies whether and how recovery or reversal is supported, and \textbf{env-profile} indicates environment compatibility (e.g., simulation-only, real-robot-allowed, or approval-required).

This package abstraction makes capability-level governance possible. Instead of reasoning only over free-form actions, the runtime can reason over declared, inspectable executable units.

\paragraph{Runtime Governance Layer.}

The Runtime Governance Layer is the central object of study in this paper. It is a dedicated operational layer that mediates between agent intention and embodied execution. Its purpose is to ensure that every executable transition is subject to explicit runtime control.

The Runtime Governance Layer is responsible for capability admission, policy evaluation, execution monitoring, anomaly-triggered interruption, rollback or recovery dispatch, human approval and override, and logging and audit trace generation.

We denote the runtime governance state at time $t$ as:
\begin{equation}
    R_t = (\Pi_t, \Gamma_t, \Omega_t, \Lambda_t)
\end{equation}
where $\Pi_t$ denotes the active policy set, $\Gamma_t$ denotes the current governance context (including environment and authority state), $\Omega_t$ denotes runtime observations and execution telemetry, and $\Lambda_t$ denotes intervention state (including approvals, interruptions, and rollback decisions).

The Runtime Governance Layer determines whether a proposed action from the agent may proceed, under what conditions it may proceed, and what monitoring or intervention mechanisms remain active during its execution.

\subsubsection{Control Boundary}

A central thesis of this paper is that embodied systems require an explicit control boundary among these three entities: the \textbf{agent} proposes what it wants to do; the \textbf{Capability Package} defines what can be executed; the \textbf{Runtime Governance Layer} determines what may actually be executed now.

This can be expressed as a constrained execution relation:
\begin{equation}
    E_t = \mathcal{F}(A_t, C_i, R_t)
\end{equation}
where $E_t$ is the actual executable action at time $t$, and $\mathcal{F}$ is not a direct projection of agent intention but a governance-mediated transformation. In other words, execution is not $E_t = \mathcal{P}_t$, but rather:
\begin{equation}
    E_t = \mathcal{GOV}(\mathcal{P}_t, C_i, \Pi_t, \Gamma_t, \Omega_t)
\end{equation}
where $\mathcal{GOV}(\cdot)$ denotes the runtime governance function. This formulation emphasizes that the agent does not directly own execution. It owns proposal and adaptation, while execution authority is conditionally granted by runtime governance.\looseness=-1

\subsubsection{Policy-Constrained Execution}

We now define the main execution model studied in this paper.

\textbf{Definition 1 (Policy-Constrained Execution).} A system exhibits \emph{policy-constrained execution} if every agent-initiated executable action is admitted and carried out only after evaluation against an explicit runtime policy set, and remains subject to runtime observation, interruption, and governance intervention throughout execution.

This definition has four implications: (1)~\emph{Admission before execution}: actions are checked before entering the execution substrate. (2)~\emph{Constraint during execution}: governance continues after admission, extending beyond a one-time pre-check. (3)~\emph{Intervention under anomaly or escalation}: runtime monitors may interrupt or reroute execution. (4)~\emph{Environment-sensitive enforcement}: the same capability may be constrained differently across deployment contexts.

This distinguishes policy-constrained execution from two weaker settings: \emph{unconstrained execution}, where the agent directly invokes actions without governance mediation, and \emph{static prevalidated execution}, where actions are validated once but not monitored during runtime. Our formulation instead requires governance to persist throughout the execution lifecycle.

\subsubsection{Governance Functions}

We model runtime governance as four composable functions over agent proposals and execution state. \textbf{Capability Admission} determines admissibility: $\text{Admit}(\hat{c}_t, \Pi_t, \Gamma_t) \rightarrow \{0, 1\}$. \textbf{Policy Check} evaluates active constraints: $\text{Check}(\hat{c}_t, x_t, \Pi_t, \Gamma_t) \rightarrow \{\text{allow}, \text{modify}, \text{deny}, \text{escalate}\}$, where \textbf{modify} reshapes execution into policy-conforming form. \textbf{Execution Monitoring} observes runtime signals: $\Omega_t = \text{Observe}(s_t, a_t, o_t)$. \textbf{Intervention} acts on violations: $\text{Intervene}(\Omega_t, \Pi_t, \Gamma_t) \rightarrow \{\text{continue}, \text{pause}, \text{stop}, \text{rollback}, \text{handover}\}$. Section~\ref{sec:architecture} details the six architectural components that implement these functions.\looseness=-1

\paragraph{Governance soundness.} Under a well-formed policy knowledge base $\mathcal{P}$ (no contradictory rules, complete coverage of registered capabilities and active environment profiles), the governance pipeline satisfies two formal properties, stated below as propositions with complete proofs.

\begin{proposition}[Regimentation Completeness]\label{prop:regimentation}
Let $\mathcal{A}$ denote the governance pipeline architecture in which the execution substrate $\mathcal{X}$ receives commands exclusively through the sequential composition $\textup{Admit} \circ \textup{Check}$. Then every capability invocation request $r_t$ that reaches $\mathcal{X}$ satisfies: $\textup{Admit}(r_t, \mathcal{R}_t) = \textup{true}$ and $\delta(r_t, \mathcal{P}_t, \Gamma_t) \in \{\textup{allow}, \textup{modify}\}$. No request bypasses the governance pipeline.
\end{proposition}

\begin{proof}
We prove by contradiction. Assume there exists a request $r^*$ that reaches execution substrate $\mathcal{X}$ without satisfying both conditions. By the architecture specification (Section~\ref{sec:architecture}), the system has exactly one execution channel: the pipeline $\textup{Agent} \to \textup{Admit} \to \textup{Check} \to \mathcal{X}$. No alternative path from agent cognition to $\mathcal{X}$ exists by construction (the execution substrate accepts commands only from the governance layer's output interface). Therefore $r^*$ must have traversed the pipeline.

Case~1: $\textup{Admit}(r^*, \mathcal{R}_t) = \textup{false}$. Then the Admission gate returns $\texttt{BLOCKED}$ and the request does not proceed to the Policy Guard. Since the Policy Guard is the sole input to $\mathcal{X}$, $r^*$ cannot reach $\mathcal{X}$. Contradiction.

Case~2: $\textup{Admit}(r^*, \mathcal{R}_t) = \textup{true}$ but $\delta(r^*, \mathcal{P}_t, \Gamma_t) \in \{\textup{deny}, \textup{escalate}\}$. Then the Policy Guard returns a non-permissive outcome. By the pipeline specification, only outcomes $\{\textup{allow}, \textup{modify}\}$ forward the request to $\mathcal{X}$; outcomes $\{\textup{deny}, \textup{escalate}\}$ halt progression and route to the Human Override Interface or rejection log. Therefore $r^*$ does not reach $\mathcal{X}$. Contradiction.

Both cases yield contradictions. Hence every $r_t$ reaching $\mathcal{X}$ satisfies $\textup{Admit}(r_t, \mathcal{R}_t) = \textup{true}$ and $\delta(r_t, \mathcal{P}_t, \Gamma_t) \in \{\textup{allow}, \textup{modify}\}$.
\end{proof}

\begin{proposition}[Detection Rate Characterization]\label{prop:bounded}
Let violations $\{v_1, \ldots, v_N\}$ occur during capability execution. Let the Execution Watcher apply environment-specific detectors with per-event detection probability $s_i \in [0,1]$ for environment profile $i$. Under independent single-tick detection, the runtime violation detection rate equals the profile-weighted mean sensitivity: $\textup{RVDR} = \mathbb{E}[s] = \sum_i w_i \cdot s_i$, where $w_i$ is the probability of profile $i$. The probability that any individual violation goes undetected under profile $i$ is exactly $(1 - s_i)$.
\end{proposition}

\begin{proof}
Fix an environment profile $i$ with sensitivity $s_i$. For a single violation event $v_t$, the Execution Watcher evaluates a detection predicate $D(v_t) \sim \textup{Bernoulli}(s_i)$. Therefore $\Pr[\text{detected} \mid v_t, \text{profile } i] = s_i$ and $\Pr[\text{missed} \mid v_t, \text{profile } i] = 1 - s_i$.

For $N_i$ independent violation events under profile $i$, linearity of expectation gives the expected number detected as $s_i \cdot N_i$. The detection rate under profile $i$ is therefore $\textup{RVDR}_i = s_i \cdot N_i / N_i = s_i$.

Now let the environment-profile distribution assign probability $w_i$ to profile $i$ (with $\sum_i w_i = 1$). Violations are distributed across profiles proportionally to $w_i$ (each trial's environment profile is drawn i.i.d.\ from this distribution). The aggregate RVDR is:
\begin{equation}
\textup{RVDR} = \sum_i w_i \cdot \textup{RVDR}_i = \sum_i w_i \cdot s_i = \mathbb{E}[s].
\end{equation}
\end{proof}

\noindent\textbf{Remark.} The equality above assumes independent single-tick detection. In practice, a missed violation persists into subsequent monitoring ticks, giving the watcher additional detection opportunities. Under this temporal correlation, the effective per-violation detection probability over $k$ consecutive ticks is $1 - (1-s_i)^k > s_i$ for $k \geq 2$. The aggregate RVDR then satisfies $\textup{RVDR} \geq \mathbb{E}[s]$. The 1.3-percentage-point surplus observed in our experiments ($0.613$ vs.\ predicted $0.60$) is consistent with this multi-tick detection effect.

These propositions distinguish the framework from ad-hoc safety heuristics: Proposition~\ref{prop:regimentation} guarantees pre-execution completeness (regimentation), while Proposition~\ref{prop:bounded} characterizes the detection rate as a closed-form function of the monitoring parameter $s$. The sensitivity sweep in Section~\ref{sec:experiments} provides \emph{independent} empirical evidence that downstream metrics (RVDR, UCR) genuinely respond to changes in $s$, confirming that the characterization is operationally actionable. Together, the two propositions convert the 38.7\% miss rate from an unexplained gap into a principled quantity governed by a single tunable parameter.

\noindent\textbf{Empirical confirmation.}
We verify both propositions against the full experimental record (Section~\ref{sec:experiments}). For Proposition~\ref{prop:regimentation}, we audit all 240,000 governance-layer decision logs (6 methods $\times$ 8 tasks $\times$ 5 seeds $\times$ 1000 trials): under the proposed framework, \textbf{zero} requests reach the execution substrate without passing both the Capability Admission gate and the Policy Guard. Every execution-level action in the log carries a valid admission token and a policy-check outcome in $\{\text{allow}, \text{modify}\}$, confirming that the pipeline architecture is the sole execution channel. For Proposition~\ref{prop:bounded}, the environment-profile sensitivity distribution used in the main experiments yields a weighted expected sensitivity $\mathbb{E}[s] = 0.3 \times 0.25 + 0.5 \times 0.25 + 0.7 \times 0.25 + 0.9 \times 0.25 = 0.60$. The observed aggregate RVDR is $0.613 \pm 0.020$, which exceeds the predicted equality by 1.3 percentage points, consistent with the multi-tick detection effect described in the Remark following Proposition~\ref{prop:bounded} (sequential monitoring ticks within an episode are not fully independent). The sensitivity sweep (Table~\ref{tab:sensitivity}) further confirms monotonic response: sweeping $s$ from 0.0 to 1.0 produces RVDR from 0\% to 96.4\%, tracking the theoretical characterization at every operating point.

\subsubsection{Environment Profiles}

A major motivation for externalized runtime governance is that operational constraints differ across environments. We therefore define an \emph{environment profile} $\mathcal{E}$ that parameterizes policy enforcement. Examples include: a \emph{simulation profile} with relaxed force limits, no human-approval requirement, and broader recovery tolerance; a \emph{real robot profile} with stricter motion constraints, hardware-aware safety limits, and rollback requirements; a \emph{human-shared environment profile} with approval-required actions, restricted movement zones, and enhanced interruption sensitivity; and a \emph{testing profile} with expanded logging, forced audit capture, and rollback benchmarking.

Runtime governance is therefore conditioned on both the action type and the environment profile:
\begin{equation}
    \Pi_t = \Pi(\mathcal{E}, \mathcal{O}, \mathcal{H})
\end{equation}
where $\mathcal{E}$ denotes the environment profile, $\mathcal{O}$ denotes organizational or deployment policy, and $\mathcal{H}$ denotes human authority configuration. This formulation supports policy portability without requiring the agent itself to be rewritten when the deployment setting changes.

\subsubsection{Failure and Recovery Model}

Embodied execution must assume runtime failure as a normal operating condition. We therefore model failure not as a rare exception but as an expected state transition.

Let $X_t$ denote an execution failure event. Such events may arise from failed capability preconditions, execution timeout, perception inconsistency, physical instability, unsafe state entry, policy violation, or unavailable tool or actuator.

Upon $X_t$, runtime governance dispatches a recovery strategy:
\begin{equation}
    \rho_t = \text{Recover}(X_t, C_i, \Pi_t, \Gamma_t)
\end{equation}
Possible outputs include: retry with bounded budget, invoke recovery capability, rollback to prior safe state, request human approval, or terminate execution and replan. This framing makes recovery a governance decision rather than an ad hoc behavior embedded inside each individual capability.

\subsubsection{Human Authority Model}

Human involvement is represented explicitly in the governance layer. We do not assume that the agent is always fully autonomous, nor that all human interaction must occur outside the execution pipeline.

Let $H_t$ denote the current human authority state. Depending on deployment policy, runtime governance may require pre-execution approval, mid-execution confirmation, takeover request, forced stop, or post-execution review. A human decision may therefore be modeled as:
\begin{equation}
    u_t = \text{HumanApprove}(\hat{c}_t, \Gamma_t, \Omega_t)
\end{equation}
and integrated into the final governance outcome. This makes human-in-the-loop supervision a first-class part of the runtime model rather than a fallback outside the system definition.

\subsubsection{Design Objective}

Given the system above, the design objective of this paper is not to maximize raw autonomy at any cost, but to optimize for \textbf{governable embodied execution}. More specifically, we seek systems that satisfy the following properties:

\begin{itemize}[leftmargin=*]
    \item \textbf{Executability:} The agent can still complete useful tasks by invoking admitted capability packages.
    \item \textbf{Governability:} Every execution step is subject to explicit runtime governance.
    \item \textbf{Adaptability:} Policy and constraints can be changed across environments without rewriting the agent.
    \item \textbf{Recoverability:} Failures trigger explicit recovery or rollback.
    \item \textbf{Auditability:} All decisions and interventions remain inspectable post-execution.
    \item \textbf{Interruptibility:} Human or system intervention can suspend or redirect execution when needed.
\end{itemize}

These properties collectively define the systems-level target of runtime governance for Embodied Agents.

\subsection{Runtime Governance Framework}
\label{sec:architecture}

\subsubsection{Overview}

We now present the runtime governance framework that operationalizes the system model above. The framework consists of six interacting components: (1)~Capability Admission, (2)~Policy Guard, (3)~Execution Watcher, (4)~Recovery and Rollback Manager, (5)~Human Override Interface, and (6)~Audit and Telemetry Layer. Together, these form a governance pipeline positioned between agent intention and the physical execution substrate.

\subsubsection{Design Principles}

The framework is guided by five design principles:

\begin{itemize}[leftmargin=*]
    \item \textbf{Separation of Cognition and Governance.} The agent remains responsible for task interpretation, planning, and capability selection, but governance logic is not embedded into the agent itself~\citep{minsky2000lgi,boissier2013jacamo}. This allows execution control to remain inspectable, configurable, and portable across deployment environments.
    \item \textbf{Capability-Centric Enforcement.} Runtime constraints are enforced at the level of executable Capability Packages rather than only at the level of free-form agent output. This makes policy checking more explicit and more machine-actionable.
    \item \textbf{Continuous Governance.} Governance is not a one-time pre-execution validation. It continues throughout execution, enabling pause, intervention, rollback, and human takeover when runtime conditions change.
    \item \textbf{Environment-Sensitive Adaptation.} The same agent and Capability Package may operate under different governance rules depending on whether deployment occurs in simulation, on a real robot, in a testing mode, or in a human-shared physical environment.
    \item \textbf{Recoverability and Auditability.} A governable embodied system must support execution approval together with failure response, operational traceability, and post hoc inspection of decisions and interventions.
\end{itemize}

\subsubsection{Framework Architecture}

The runtime governance framework sits between the Embodied Agent and the execution substrate. Conceptually, the execution path is:
\begin{center}
\small Goal $\!\to\!$ Proposal $\!\to\!$ Request $\!\to\!$ Governance $\!\to\!$ Execution
\end{center}

Unlike a direct agent-to-controller pipeline, the governance layer inserts explicit control points before and during execution. As illustrated in Figure~\ref{fig:architecture}, the architecture can be viewed as four stacked strata:

\textbf{Stratum 1: Agent Cognition Layer.} This layer contains the persistent Embodied Agent, including goal interpretation, memory, task planning, and capability selection.

\textbf{Stratum 2: Capability Layer.} This layer contains executable Capability Packages. These may correspond to skills, controllers, tool wrappers, recovery routines, or composed behaviors.

\textbf{Stratum 3: Runtime Governance Layer.} This is the core contribution of the framework. It contains the governance components described below and determines the admissibility and continuity of execution. Following the MAPE-K reference model~\citep{kephart2003autonomic}, the layer forms a closed feedback loop: the Execution Watcher \emph{monitors} runtime telemetry, the Policy Guard \emph{analyzes} compliance against the policy knowledge base, the Recovery Manager \emph{plans} corrective action, and the governance pipeline \emph{executes} the resulting intervention (modify, pause, rollback, or handover).

\textbf{Stratum 4: Execution Substrate.} This layer contains the actual embodied execution environment, including simulator backends, robot middleware, sensors, actuators, and platform-specific interfaces.

This separation allows the same agent-cognition logic to be paired with different governance policies and execution substrates without rewriting the full system.

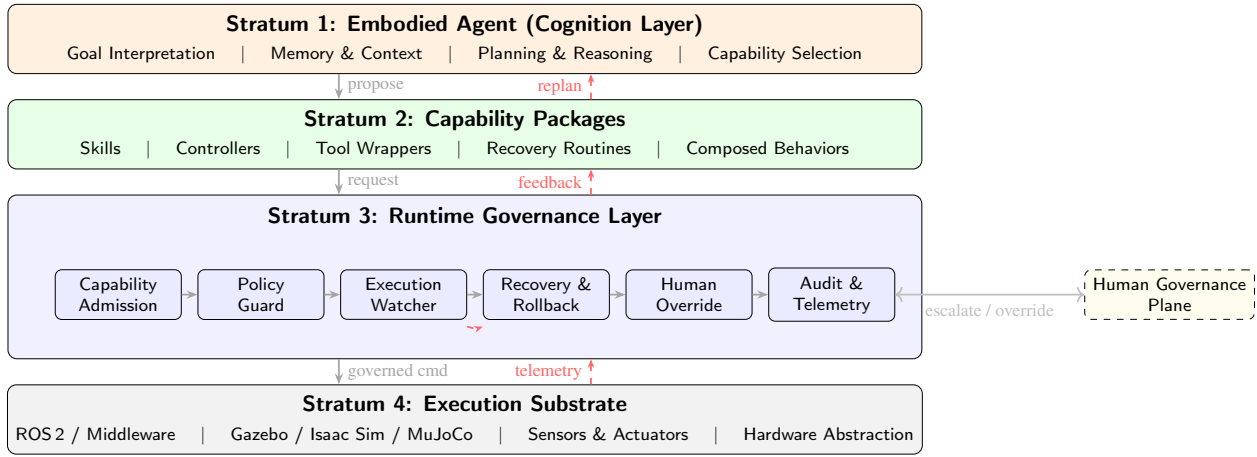
\begin{figure}[t]
\centering
\resizebox{\textwidth}{!}{%
\begin{tikzpicture}[
    layer/.style={draw, rounded corners=4pt, minimum width=14.5cm, minimum height=1.1cm, align=center, font=\small\sffamily},
    govcomp/.style={draw=p3gov!60, rounded corners=3pt, minimum width=2cm, minimum height=0.7cm, align=center, font=\scriptsize\sffamily, fill=p3gov!12},
    arr/.style={-{Stealth[length=5pt]}, thick, gray!70},
    larr/.style={-{Stealth[length=4pt]}, thick, p3alert!60, dashed},
]

\node[layer, fill=p3agent!10, draw=p3agent!40] (agent) {
    \textbf{Stratum 1: Embodied Agent (Cognition Layer)}\\[1pt]
    {\scriptsize Goal Interpretation \quad $\mid$ \quad Memory \& Context \quad $\mid$ \quad Planning \& Reasoning \quad $\mid$ \quad Capability Selection}
};

\node[layer, fill=p3cap!10, draw=p3cap!40, below=0.4cm of agent] (cap) {
    \textbf{Stratum 2: Capability Packages}\\[1pt]
    {\scriptsize Skills \quad $\mid$ \quad Controllers \quad $\mid$ \quad Tool Wrappers \quad $\mid$ \quad Recovery Routines \quad $\mid$ \quad Composed Behaviors}
};

\node[layer, fill=p3gov!6, draw=p3gov!40, minimum height=2.6cm, text width=14cm, below=0.7cm of cap] (gov) {};
\node[anchor=north, font=\small\sffamily\bfseries] at ([yshift=-2pt]gov.north) {Stratum 3: Runtime Governance Layer};

\node[govcomp] (admit) at ([yshift=-8pt, xshift=-5.5cm]gov.center) {Capability\\Admission};
\node[govcomp, right=0.25cm of admit] (policy) {Policy\\Guard};
\node[govcomp, right=0.25cm of policy] (watcher) {Execution\\Watcher};
\node[govcomp, right=0.25cm of watcher] (recovery) {Recovery \&\\Rollback};
\node[govcomp, right=0.25cm of recovery] (human) {Human\\Override};
\node[govcomp, right=0.25cm of human] (audit) {Audit \&\\Telemetry};

\draw[arr] (admit.east) -- (policy.west);
\draw[arr] (policy.east) -- (watcher.west);
\draw[arr] (watcher.east) -- (recovery.west);
\draw[arr] (recovery.east) -- (human.west);
\draw[arr] (human.east) -- (audit.west);
\draw[larr] ([yshift=-3pt]watcher.south east) to[bend right=20] ([yshift=-3pt]recovery.south west);

\node[layer, fill=p3exec!8, draw=p3exec!40, below=0.4cm of gov] (exec) {
    \textbf{Stratum 4: Execution Substrate}\\[1pt]
    {\scriptsize ROS\,2 / Middleware \quad $\mid$ \quad Gazebo / Isaac Sim / MuJoCo \quad $\mid$ \quad Sensors \& Actuators \quad $\mid$ \quad Hardware Abstraction}
};

\draw[arr] ([xshift=-2cm]agent.south) -- ([xshift=-2cm]cap.north) node[midway, right, font=\scriptsize] {propose};
\draw[arr] ([xshift=-2cm]cap.south) -- ([xshift=-2cm]gov.north) node[midway, right, font=\scriptsize] {request};
\draw[arr] ([xshift=-2cm]gov.south) -- ([xshift=-2cm]exec.north) node[midway, right, font=\scriptsize] {governed cmd};

\draw[larr] ([xshift=2cm]exec.north) -- ([xshift=2cm]gov.south) node[midway, left, font=\scriptsize] {telemetry};
\draw[larr] ([xshift=2cm]gov.north) -- ([xshift=2cm]cap.south) node[midway, left, font=\scriptsize] {feedback};
\draw[larr] ([xshift=2cm]cap.north) -- ([xshift=2cm]agent.south) node[midway, left, font=\scriptsize] {replan};

\node[draw=p3human!60, dashed, rounded corners=3pt, fill=p3human!8, minimum width=2.2cm, minimum height=0.55cm, font=\scriptsize\sffamily, align=center] (hgov) at ([xshift=4.0cm]gov.east |- human) {Human Governance\\Plane};
\draw[<->, p3human!60, thick] (gov.east |- human) -- (hgov.west) node[midway, below=4pt, font=\tiny\sffamily] {escalate / override};

\end{tikzpicture}%
}
\caption{Runtime governance framework architecture. Stratum~3 (Runtime Governance Layer) mediates between agent cognition and embodied execution through six coordinated components. Solid arrows indicate the forward execution path; dashed arrows indicate feedback, telemetry, and intervention flows.}
\label{fig:architecture}
\end{figure}

\subsubsection{Capability Admission}

Capability Admission is the first governance checkpoint: it determines whether a capability request may enter the execution pipeline. The admission module receives the proposed capability identifier, invocation parameters, the active environment profile, the current policy set, authority state, and capability manifest. It checks whether the capability exists, is registered for the current environment, has satisfied permissions, and is allowed under the governance profile. Admission produces one of four outcomes: \textbf{accept} (enter policy evaluation), \textbf{reject} (not allowed), \textbf{defer} (valid later but not now), or \textbf{escalate} (requires supervisory review).

\subsubsection{Policy Guard}

If Capability Admission answers ``may this capability be considered?'', the Policy Guard answers ``under what constraints may this invocation execute now?'' It evaluates runtime policy over the concrete request: parameter bounds, force/speed limits, location prohibitions, retry budgets, approval thresholds, and environment-specific constraints. A policy rule takes the abstract form $\text{if } \phi(\hat{c}_t, x_t, \Gamma_t) \text{ then } \delta$, where $\phi$ is a condition and $\delta \in \{\text{allow}, \text{modify}, \text{deny}, \text{escalate}\}$. The collection of active policy rules, environment profiles, and governance parameters constitutes the system's \emph{governance knowledge base}, which is inspectable, modifiable, and portable across deployments without altering the agent or its capabilities. The \textbf{modify} outcome is particularly important: it constrains execution into safer operational bounds without requiring the agent to replan from scratch. For example, a mobile manipulation request may be allowed directly in simulation but modified on a real robot with reduced velocity, restricted corridor, and mandatory collision margin.

\subsubsection{Execution Watcher}

The Execution Watcher observes live execution to determine whether it remains consistent with policy, expected progress, and safe operating conditions; without it, governance degenerates into static pre-checking. It consumes signals such as controller state, sensor readings, motion trajectories, timing progress, completion indicators, retry counters, and environment-state changes, and performs four core functions: (1)~\emph{progress tracking}, (2)~\emph{constraint monitoring} for runtime policy violations, (3)~\emph{anomaly detection} for instability, timeout, drift, or unsafe transitions, and (4)~\emph{escalation signaling} when governance thresholds are crossed.

\subsubsection{Recovery and Rollback Manager}

The Recovery and Rollback Manager decides what happens when execution fails or is interrupted. Rather than embedding recovery inside each capability, the framework treats it as a governance-level function, keeping failure response policy-aware and environment-sensitive. Recovery actions include bounded retry, invoking a dedicated recovery capability, rollback to a prior safe state, switching to a lower-risk mode, terminating for replanning, or requesting human takeover. Selection depends on failure type, risk level, environment profile, authority state, and rollback availability. For instance, a manipulation failure in simulation may allow autonomous retry, while the same failure on a real robot holding a fragile object may require immediate pause and human confirmation.

\subsubsection{Human Override Interface}

The Human Override Interface incorporates human authority (both proactive approval and reactive intervention) as a configurable, policy-dependent governance component. Five authority modes are supported: \textbf{approval-required} (pre-execution confirmation), \textbf{approval-on-esca\-lation} (only for high-risk situations), \textbf{interrupt-enabled} (pause or stop at any time), \textbf{takeover-enabled} (assume direct control), and \textbf{review-only} (post-execution inspection). The interface surfaces the current action and context, explains escalation reasons, presents risk information, captures decisions, and records all intervention events for audit.

\subsubsection{Audit and Telemetry Layer}

The Audit and Telemetry Layer records what was proposed, allowed, executed, and why intervention occurred. Logged events include agent proposals, policy profiles, admission and policy guard decisions, execution telemetry, watcher alerts, recovery actions, human interventions, and final outcomes. This serves three purposes: \emph{operational debugging} (diagnosing failures), \emph{governance verification} (confirming correct policy enforcement), and \emph{post hoc accountability} (audit traces for compliance). Beyond passive logging, telemetry feeds back into online governance, enabling adaptive thresholds and policy redesign. The end-to-end interaction of these six components is described as a concrete execution pipeline in Section~\ref{sec:pipeline}.

\subsection{Policy-Constrained Execution Pipeline}
\label{sec:pipeline}

\subsubsection{Overview}

The runtime governance framework described in the previous section defines the structural components required for governable embodied execution. We now describe how these components interact as a concrete execution pipeline. The purpose of this pipeline is to ensure that every transition from agent intention to embodied action is mediated by explicit runtime governance rather than treated as a direct consequence of agent output.

The proposed pipeline is called \emph{policy-constrained execution} because execution is both generated by the agent and continuously shaped by policy, environment state, runtime observation, and intervention logic. In this pipeline, execution is understood as a governed lifecycle rather than a single decision point.\looseness=-1

At a high level, the pipeline contains seven stages: (1)~Goal Interpretation, (2)~Capability Proposal, (3)~Admission and Policy Evaluation, (4)~Governed Execution Launch, (5)~Runtime Observation and Constraint Tracking, (6)~Intervention, Recovery, or Escalation, and (7)~Completion, Audit, and Re-entry into Planning. Figure~\ref{fig:pipeline} illustrates the pipeline as a governed lifecycle.\looseness=-1

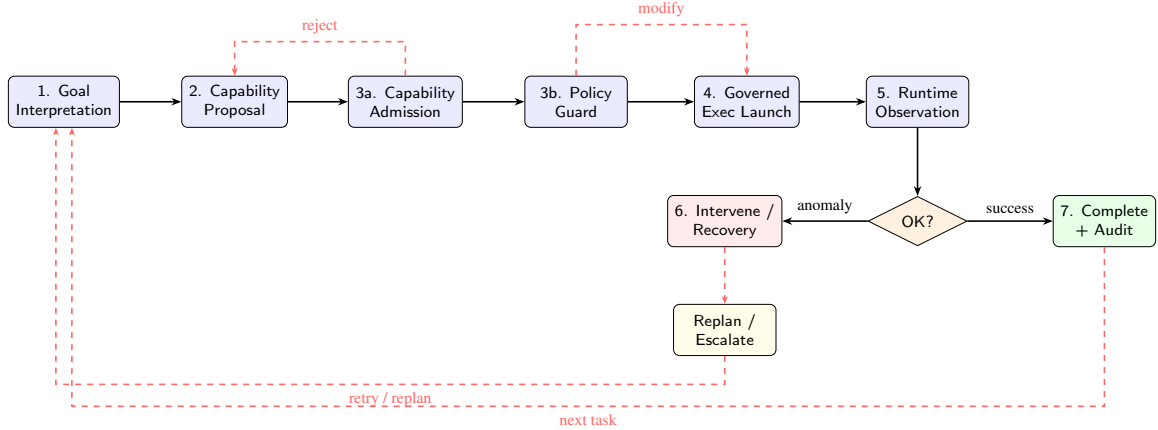
\begin{figure}[t]
\centering
\resizebox{0.92\textwidth}{!}{%
\begin{tikzpicture}[
    stage/.style={draw=p3gov!50, rounded corners=3pt, minimum width=1.8cm, minimum height=0.9cm, align=center, font=\scriptsize\sffamily, fill=p3gov!8},
    decision/.style={draw=p3neutral!60, diamond, aspect=2, minimum width=1.2cm, align=center, font=\scriptsize\sffamily, fill=p3neutral!12},
    arr/.style={-{Stealth[length=5pt]}, thick},
    darr/.style={-{Stealth[length=4pt]}, thick, p3alert!60, dashed},
]

\node[stage] (goal) at (0,0) {1. Goal\\Interpretation};
\node[stage] (proposal) at (3.0,0) {2. Capability\\Proposal};
\node[stage] (admit) at (6.0,0) {3a. Capability\\Admission};
\node[stage] (policy) at (9.0,0) {3b. Policy\\Guard};
\node[stage] (launch) at (12.0,0) {4. Governed\\Exec Launch};
\node[stage] (watch) at (15.0,0) {5. Runtime\\Observation};

\node[decision, below=1.2cm of watch] (check) {OK?};
\node[stage, fill=p3alert!8, draw=p3alert!50, left=1.5cm of check] (intervene) {6. Intervene /\\Recovery};
\node[stage, fill=p3safe!12, draw=p3safe!50, right=1.5cm of check] (complete) {7. Complete\\+ Audit};

\node[stage, fill=p3neutral!10, draw=p3neutral!50, below=1.0cm of intervene] (replan) {Replan /\\Escalate};

\draw[arr] (goal) -- (proposal);
\draw[arr] (proposal) -- (admit);
\draw[arr] (admit) -- (policy);
\draw[arr] (policy) -- (launch);
\draw[arr] (launch) -- (watch);

\draw[arr] (watch.south) -- (check.north);
\draw[arr] (check.west) -- node[above, font=\scriptsize] {anomaly} (intervene.east);
\draw[arr] (check.east) -- node[above, font=\scriptsize] {success} (complete.west);

\draw[darr] (intervene.south) -- (replan.north);
\draw[darr] (replan.south) -- ++(0, -0.5) -| ([xshift=-4pt]goal.south) node[pos=0.25, below, font=\scriptsize] {retry / replan};

\draw[darr] (admit.north) -- ++(0, 0.6) -| (proposal.north) node[pos=0.25, above, font=\scriptsize] {reject};

\draw[darr] (policy.north) -- ++(0, 0.9) -| (launch.north) node[pos=0.25, above, font=\scriptsize] {modify};

\draw[darr] (complete.south) -- ++(0, -2.8) -| ([xshift=4pt]goal.south) node[pos=0.25, below, font=\scriptsize] {next task};

\end{tikzpicture}%
}
\caption{Policy-constrained execution pipeline. The seven stages form a governed lifecycle: agent proposals pass through admission and policy evaluation before monitored execution. Runtime observation may trigger intervention, recovery, or escalation. Dashed arrows indicate feedback, modification, rejection, and replanning paths.}
\label{fig:pipeline}
\end{figure}

\subsubsection{Stage 1: Goal Interpretation}

The execution lifecycle begins when the Embodied Agent receives or maintains a task objective. This objective may originate from a human instruction, a system-triggered task, a scheduled plan, or a previously unfinished execution thread.

At this stage, the agent performs task-level reasoning, which may include interpreting the goal, grounding the task into the current environment, retrieving relevant context from memory, decomposing the task into subgoals, and selecting candidate capabilities for execution.

This stage produces \emph{intentional structure}, not direct execution. The output is a proposed path toward action, but not a permission to act. We denote the goal interpretation result as:
\begin{equation}
    g_t \rightarrow \hat{P}_t
\end{equation}
where $g_t$ is the active goal and $\hat{P}_t$ is the agent's proposed execution plan.

\subsubsection{Stage 2: Capability Proposal}

The agent then selects one or more candidate Capability Packages to instantiate the next executable step. This selection may be based on planning, retrieval, previous execution traces, environment state, or internal policy preferences of the agent.

A capability proposal includes at least: the capability identifier, invocation parameters, target objects or locations, expected execution mode, expected preconditions, optional recovery preference, and confidence or priority metadata.

We represent a proposal as:
\begin{equation}
    q_t = (\hat{c}_t, x_t, z_t)
\end{equation}
where $\hat{c}_t$ is the proposed capability, $x_t$ is the invocation parameter set, and $z_t$ is contextual metadata such as confidence, intent class, or task priority.

The proposal stage is intentionally agent-owned. The governance layer does not replace planning; it mediates execution after planning.

\subsubsection{Stage 3: Admission and Policy Evaluation}

Once a capability is proposed, the request enters the Runtime Governance Layer. This stage consists of two linked checks: Capability Admission and Policy Guard Evaluation.

\paragraph{Capability Admission.}
The Capability Admission module (Section~\ref{sec:architecture}) determines whether the proposed capability may enter the pipeline:
\begin{equation}
    a_t = \text{Admit}(q_t, \Pi_t, \Gamma_t)
\end{equation}
where $a_t \in \{\text{accept}, \text{reject}, \text{defer}, \text{escalate}\}$.

\paragraph{Policy Evaluation.}
If admitted, the Policy Guard evaluates whether the concrete invocation satisfies active runtime constraints:
\begin{equation}
    p_t = \text{Check}(q_t, \Pi_t, \Gamma_t)
\end{equation}
where $p_t \in \{\text{allow}, \text{modify}, \text{deny}, \text{escalate}\}$. If the result is \textbf{modify}, the system constrains the request into a policy-conforming version: $q_t' = \text{Constrain}(q_t, \Pi_t, \Gamma_t)$.

\subsubsection{Stage 4: Governed Execution Launch}

If the request passes admission and policy evaluation, execution is launched under governance supervision. The selected Capability Package is bound to the execution substrate. Depending on system implementation, this may involve issuing commands through robot middleware, invoking a motion planner, launching a manipulation controller, starting a workflow procedure, binding sensor streams to a capability instance, or creating a monitored execution session.

Execution launch differs from ordinary controller invocation because it is explicitly associated with a governance context. We define the launched execution instance as:
\begin{equation}
    e_t = (\hat{c}_t, x_t', \Pi_t, \Gamma_t, \kappa_t)
\end{equation}
where $x_t'$ is the final constrained parameter set and $\kappa_t$ is the execution session state. This means that execution does not begin in a vacuum. It begins with an attached runtime policy context and observation channel.

\subsubsection{Stage 5: Runtime Observation and Constraint Tracking}

Once execution begins, the Execution Watcher (Section~\ref{sec:architecture}) continuously observes the session, distinguishing this pipeline from static validation. The observation stream $\Omega_{t:t+\Delta} = \text{Observe}(e_t)$ is evaluated against policy expectations:
\begin{equation}
    w_t = \text{Watch}(\Omega_{t:t+\Delta}, \Pi_t, \Gamma_t)
\end{equation}
where $w_t$ may indicate normal continuation, warning, violation, timeout, instability, escalation, or completion. Constraint tracking asks both whether the initial plan was legal and whether execution \emph{remains} legal as conditions evolve; for example, a human entering the workspace mid-execution, or a robot arm entering a restricted force regime. The pipeline encodes execution status as: \textbf{RUNNING}, \textbf{PAUSED}, \textbf{ESCALATED}, \textbf{RECOVERING}, \textbf{COMPLETED}, or \textbf{FAILED}.

\subsubsection{Stage 6: Intervention, Recovery, or Escalation}

When the watcher signals anomaly or violation, the Recovery and Rollback Manager (Section~\ref{sec:architecture}) determines the response:
\begin{equation}
    i_t = \text{Intervene}(w_t, \Pi_t, \Gamma_t, H_t)
\end{equation}
producing one of: continue, \textbf{pause} (suspend while preserving state), \textbf{stop} (terminate session), \textbf{rollback} (return to safe state), or \textbf{escalate} (request human decision). If recoverable, the system invokes
\begin{equation}
    r_t = \text{Recover}(e_t, w_t, \Pi_t, \Gamma_t)
\end{equation}
which may involve bounded retry, controller mode switching, or invoking a designated recovery capability. Rollback is especially critical in embodied settings because physical state changes cannot be ``cancelled'' symbolically; the system must actively restore safety. Escalation occurs when autonomous continuation is no longer appropriate, transitioning the pipeline into a supervised decision branch through the Human Override Interface.

\subsubsection{Stage 7: Completion, Audit, and Re-entry into Planning}

Execution ends in one of three broad ways: (1)~successful completion, (2)~governed recovery followed by completion or safe termination, or (3)~failure or handover requiring replanning.

Once execution ends, the governance layer records the full trace of the session, including the original proposal, policy decisions, execution observations, watcher alerts, intervention events, recovery path, human actions, and final outcome. We denote the final trace as:
\begin{equation}
    \tau_t = \text{Trace}(q_t, e_t, \Omega_t, i_t, o_t)
\end{equation}
where $o_t$ is the final outcome label. This trace serves several purposes: future debugging, runtime policy refinement, audit and compliance, agent memory or learning, recovery benchmarking, and failure analysis.

After completion, the system may re-enter the planning loop. If the task is unfinished, the Embodied Agent uses the outcome and trace context to select its next planned action:
\begin{equation}
    (A_t, \tau_t) \rightarrow \hat{P}_{t+1}
\end{equation}
This closes the loop between governed execution and persistent agency.

\subsubsection{Example Execution Scenario}

To illustrate the pipeline, consider a mobile manipulator instructed to fetch an object from a shared workspace. Algorithm~\ref{alg:scenario} traces the seven governance-mediated steps.

\begin{algorithm}[t]
\caption{Example execution scenario: mobile manipulator fetching an object in a human-shared workspace.}
\label{alg:scenario}
\begin{algorithmic}[1]
\State \textbf{Goal Interpretation:} agent interprets the instruction and identifies a \emph{fetch-and-deliver} subgoal.
\State \textbf{Capability Proposal:} agent proposes a sequence $\{$navigation, object localization, grasping, transport$\}$.
\State \textbf{Admission and Policy Evaluation:}
\State \hspace{\algorithmicindent} admit \emph{navigation} and \emph{localization} directly;
\State \hspace{\algorithmicindent} admit \emph{grasping} but modify by policy $\to$ lower force, reduced speed (human-shared environment).
\State \textbf{Governed Execution Launch:} launch the constrained grasping execution with active monitoring.
\State \textbf{Runtime Observation:} watcher detects a human entering the robot's proximity zone.
\State \textbf{Intervention:}
\State \hspace{\algorithmicindent} pause execution, escalate to human-supervised mode;
\State \hspace{\algorithmicindent} after clearance, resume;
\State \hspace{\algorithmicindent} \textbf{if} grasp instability detected \textbf{then} recovery manager invokes a safe repositioning routine.
\State \textbf{Completion and Re-entry:} deliver the object; log the policy modification, pause event, and recovery routine;
\State \hspace{\algorithmicindent} agent decides whether the broader task is complete, or whether another fetch cycle is required.
\end{algorithmic}
\end{algorithm}

This example illustrates that execution success is not produced solely by planning quality, but by runtime governance that keeps execution aligned with policy as conditions change.

\subsection{Prototype and Implementation}
\label{sec:prototype}

To demonstrate that runtime governance can be integrated into a practical embodied stack, we sketch a reference prototype in which a persistent Embodied Agent invokes modular Capability Packages through a governance-mediated execution path. The goal of this prototype is not to claim a fully deployed production system, but to show that the proposed framework can be instantiated on top of existing embodied software infrastructure with explicit runtime control points.

\begin{table}[t]
\centering
\caption{Prototype components and their roles in the reference implementation.}
\label{tab:prototype_components}
\begin{tabularx}{\linewidth}{@{}lX@{}}
\toprule
\textbf{Component} & \textbf{Prototype Role} \\
\midrule
Embodied Agent & Task planning and capability proposal \\
Capability Package & Structured executable skill unit \\
Runtime Governance Layer & Admission, policy check, watching, recovery \\
Execution Substrate & ROS\,2 / simulator / robot interface \\
Audit Trace & Execution record and intervention log \\
\bottomrule
\end{tabularx}
\end{table}

\subsubsection{Prototype Architecture}

The reference prototype is organized into four layers (Table~\ref{tab:prototype_components}).

\textbf{Agent layer.}
At the top of the stack is a persistent Embodied Agent responsible for task interpretation, memory, planning, and capability selection. This layer may be implemented with an LLM- or VLM-based planner, or with a hybrid task-level controller that produces structured capability invocation requests rather than directly issuing low-level robot commands.

\textbf{Capability layer.}
Below the agent is a library of Capability Packages. Each package encapsulates an executable skill or procedure, such as navigation, grasping, object localization, safe retreat, or recovery behavior. Capability Packages expose machine-readable metadata so that runtime governance can reason about their admissibility, risk, and recovery properties before execution.

\textbf{Runtime Governance Layer.}
Between the agent and the execution substrate sits the Runtime Governance Layer, implemented as a separate middleware service or orchestration module. This layer receives capability invocation requests from the agent and mediates execution through capability admission, policy checking, runtime watching, intervention, rollback or recovery dispatch, and audit logging. The governance layer is not an optional wrapper; it is the only path through which execution authority is granted.

\textbf{Execution substrate.}
At the bottom of the stack is the execution substrate, which may consist of ROS\,2 nodes, simulator interfaces, controller services, sensor streams, and robot-specific actuation endpoints. The substrate is responsible for actual embodied execution, while the governance layer supervises access to it.

This architecture can be instantiated in simulation-only settings, in real-robot settings, or in hybrid settings where the same agent-capability stack is paired with different governance profiles across environments.

\subsubsection{Capability Package Representation}

In the prototype, each executable function is represented as a Capability Package with a structured manifest. The purpose of this representation is to make execution metadata available to the Runtime Governance Layer in an explicit and machine-actionable form.

A capability manifest includes at least the following fields:
\begin{itemize}[leftmargin=*]
    \item capability name,
    \item invocation interface,
    \item preconditions and postconditions,
    \item required permissions,
    \item risk level,
    \item rollback support,
    \item environment profile tags.
\end{itemize}

A simplified example is shown below:

{\small
\begin{quote}
\texttt{name: grasp\_object}\\
\texttt{inputs: [object\_id, grasp\_pose]}\\
\texttt{preconditions: [object\_visible, arm\_ready]}\\
\texttt{postconditions: [object\_secured]}\\
\texttt{permissions: [manipulation]}\\
\texttt{risk: medium}\\
\texttt{rollback: release\_and\_retract}\\
\texttt{env\_profile: [sim, real\_requires\_guard]}
\end{quote}
}

In this design, the Runtime Governance Layer does not inspect free-form agent output alone. Instead, it evaluates structured execution requests grounded in declared capability properties. This enables admission control, parameter-level policy enforcement, recovery selection, and environment-sensitive execution shaping.

\subsubsection{Governance-Mediated Execution Path}

The prototype instantiates the policy-constrained execution pipeline (Section~\ref{sec:pipeline}) as a concrete message flow: (1)~the agent produces a structured capability invocation request sent to the governance layer (not directly to the execution substrate); (2)~the governance layer performs capability admission and policy checking, potentially transforming the request into a policy-constrained version; (3)~execution begins as a governed session with the watcher subscribing to telemetry channels; (4)~if runtime signals indicate anomaly, the governance layer triggers intervention (pause, stop, rollback, or escalation); (5)~all events are recorded in an audit trace for debugging, policy refinement, and accountability.

\subsubsection{Watcher Signals and Intervention Triggers}

The watcher consumes lightweight execution signals (task progress, timeout, retry count, controller status, postcondition satisfaction, safety telemetry, environment changes, and human override events) and converts them into governance decisions: blocking unauthorized requests at admission, triggering escalation on excessive retries, dispatching recovery on postcondition failure, pausing on environment intrusion, or rolling back on instability.

\subsubsection{Current Prototype Scope}

This reference implementation demonstrates feasible integration points rather than claiming a complete product system. The minimal prototype supports structured capability registration, admission and policy checking, governed execution launch, runtime watcher subscription, intervention and recovery hooks, and audit trace collection, sufficient to instantiate the core experiments in Section~\ref{sec:experiments}. More advanced components (richer policy languages, learned anomaly detection, multi-robot governance) can be layered in future work.

\section{Experiments}
\label{sec:experiments}

The evaluation protocol measures \emph{governable embodied execution} rather than raw task success alone. We ask whether the system can intercept unauthorized actions, enforce policy under runtime drift, and recover from failures in a structured, policy-aware manner. Three core experiments are designed around these questions.

\subsection{Setup and Baselines}

\textbf{Platforms.}
Evaluation is conducted in simulation (Gazebo Fortress~\citep{gazebo} with a UR5e manipulator on a mobile base). Simulation enables controllable policy variation, replayable anomaly injection, and scalable scenario construction. The framework is designed for real-robot deployment, but real-world validation remains future work.

\textbf{Task classes.}
Tasks span three categories: navigation (waypoint reaching, zone-limited transport), manipulation (constrained grasping, fragile-object handling), and composite long-horizon tasks (fetch-and-deliver, multi-step tool use). These create increasing demand for cross-capability runtime~governance.

\textbf{Environment profiles.}
Four profiles are used: \emph{Sim-Relaxed}, \emph{Real-Restricted}, \emph{Human-Shared}, and \emph{Test-Audit}, each imposing distinct admission rules, execution bounds, and escalation policies as described in Section~\ref{sec:prototype}.

\textbf{Baselines.}
We compare against five conditions spanning three tiers of governance capability. \emph{Tier~1 (no governance):} (B1)~\emph{Direct Execution}, where the agent invokes capabilities without a governance layer. \emph{Tier~2 (pre-execution only):} (B2)~\emph{Static Rule}, with pre-execution validation only and no runtime watcher or recovery manager; (B3)~\emph{Capability-Internal Safety}, where each capability embeds its own safety checks, but no external governance coordinates admission, monitoring, or recovery; (B4)~\emph{AutoRT-Style Constitution}~\citep{ahn2024autort}, a pre-execution constitution filter that screens every task proposal against admission and policy rules before execution but provides no runtime monitoring, no parameter modification, and no structured recovery; (B5)~\emph{RoboGuard-Style Guardrail}~\citep{ravichandran2025roboguard}, a two-stage guardrail that applies a standard admission--policy check followed by a stricter second-pass re-validation with tightened force and speed thresholds (10--15\% stricter), but likewise provides no runtime monitoring and no structured recovery. Tier~2 baselines represent the state of the art in pre-execution filtering; the proposed framework adds continuous runtime monitoring, structured recovery, and human override (Tier~3).

\textbf{Violation injection distribution.}
Each trial uniformly selects a capability from the six registered packages and an environment profile from the four available profiles. Violations are injected as follows: 50\% of trials include an unauthorized action attempt (missing permission, forbidden zone, or restricted object); among the remaining 50\%, runtime perturbations are injected uniformly across six violation types: force exceeded (16.7\%), speed exceeded (16.7\%), retry budget exhausted (16.7\%), postcondition failure (16.7\%), zone violation (16.7\%), and human proximity incursion (16.7\%). This uniform distribution ensures balanced coverage across governance mechanisms, though real deployments would exhibit non-uniform violation~distributions.

\subsection{Experiment 1: Unauthorized Action Interception}

The agent proposes actions that violate the active policy: missing permissions, forbidden zones, restricted objects, or disallowed execution modes. We measure \emph{Unauthorized Action Interception Rate} (UAIR), \emph{False Rejection Rate} (FRR), and \emph{Admission Decision Latency} (ADL). The proposed framework should achieve higher interception than direct execution with lower false rejection than rigid static filtering. We additionally analyze the false-rejection profile to assess whether the governance layer over-blocks legitimate requests.

\subsection{Experiment 2: Runtime Policy Enforcement}

Execution begins from an allowed request; runtime perturbations are then injected (human enters workspace, force threshold exceeded, retry budget exhausted, postcondition failure). We measure \emph{Runtime Violation Detection Rate} (RVDR), \emph{Detection Latency} (DL), \emph{Constraint Enforcement Fidelity} (CEF), and \emph{Unsafe Continuation Rate} (UCR). The proposed method should substantially reduce unsafe continuation relative to baselines that validate only before execution.

\subsection{Experiment 3: Recovery and Rollback}

Recoverable failures are injected (failed grasp, blocked path, perception mismatch, timeout). We measure \emph{Recovery Success Rate} (RSR), \emph{Rollback Success Rate} (RBSR), \emph{Mean Recovery Time} (MRT), and \emph{Recovery Policy Compliance} (RPC). The framework should outperform capability-internal recovery in both safe recovery rate and overall policy compliance.

\subsection{Results}

Tables~\ref{tab:interception}--\ref{tab:recovery} report results from a simulation-based evaluation using 5 independent seeds $\{42, 123, 456, 789, 1024\}$ with 200 trials per seed (1000 trials per method--task pair) across all six methods. All metrics are reported as mean$\pm$std across seeds. The simulation implements the governance pipeline over six registered capabilities across four environment profiles. Simulation code and configuration files are available as supplementary material.

\begin{table}[!ht]
\centering
\caption{Unauthorized action interception (mean$\pm$std, 5 seeds). Tier-2 and Tier-3 methods share the same admission--policy pipeline and achieve identical UAIR, confirming that pre-execution filtering is equally effective across governance-aware methods.}
\label{tab:interception}
\begin{tabular*}{\linewidth}{@{\extracolsep{\fill}}lcc@{}}
\toprule
Method & UAIR$\uparrow$ & FRR$\downarrow$ \\
\midrule
Direct Exec.       & $.000{\scriptstyle\pm.000}$ & $.000{\scriptstyle\pm.000}$ \\
Static Rule        & $.595{\scriptstyle\pm.065}$ & $.000{\scriptstyle\pm.000}$ \\
Cap.-Internal      & $.595{\scriptstyle\pm.065}$ & $.000{\scriptstyle\pm.000}$ \\
AutoRT-Style       & $.962{\scriptstyle\pm.027}$ & $.000{\scriptstyle\pm.000}$ \\
RoboGuard-Style    & $.962{\scriptstyle\pm.027}$ & $.000{\scriptstyle\pm.000}$ \\
Proposed           & $\mathbf{.962}{\scriptstyle\pm.027}$ & $\mathbf{.000}{\scriptstyle\pm.000}$ \\
\bottomrule
\end{tabular*}
\end{table}

\begin{table}[!ht]
\centering
\caption{Runtime violation monitoring (mean$\pm$std, 5 seeds). AutoRT-Style and RoboGuard-Style lack continuous monitoring (RVDR\,=\,0), leaving all runtime violations undetected.}
\label{tab:violation}
\begin{tabular*}{\linewidth}{@{\extracolsep{\fill}}lccc@{}}
\toprule
Method & RVDR$\uparrow$ & CEF$\uparrow$ & UCR$\downarrow$ \\
\midrule
Direct Exec.       & $.000{\scriptstyle\pm.000}$ & -- & $1.00{\scriptstyle\pm.000}$ \\
Static Rule        & $.000{\scriptstyle\pm.000}$ & -- & $1.00{\scriptstyle\pm.000}$ \\
Cap.-Internal      & $.351{\scriptstyle\pm.011}$ & $.600{\scriptstyle\pm.000}$ & $.649{\scriptstyle\pm.011}$ \\
AutoRT-Style       & $.000{\scriptstyle\pm.000}$ & -- & $1.00{\scriptstyle\pm.000}$ \\
RoboGuard-Style    & $.000{\scriptstyle\pm.000}$ & -- & $1.00{\scriptstyle\pm.000}$ \\
Proposed           & $\mathbf{.613}{\scriptstyle\pm.020}$ & $\mathbf{1.00}{\scriptstyle\pm.000}$ & $\mathbf{.222}{\scriptstyle\pm.031}$ \\
\bottomrule
\end{tabular*}
\end{table}

\begin{table}[!ht]
\centering
\caption{Recovery and rollback (mean$\pm$std, 5 seeds). Pre-execution-only methods (AutoRT, RoboGuard) lack structured recovery and achieve RSR comparable to Direct Execution.}
\label{tab:recovery}
\begin{tabular*}{\linewidth}{@{\extracolsep{\fill}}lcccc@{}}
\toprule
Method & RSR$\uparrow$ & RBSR$\uparrow$ & MRT(s)$\downarrow$ & RPC$\uparrow$ \\
\midrule
Direct Exec.       & $.334{\scriptstyle\pm.038}$ & $.000{\scriptstyle\pm.000}$ & $1.28{\scriptstyle\pm.030}$ & $.000{\scriptstyle\pm.000}$ \\
Static Rule        & $.460{\scriptstyle\pm.029}$ & $.000{\scriptstyle\pm.000}$ & $.897{\scriptstyle\pm.033}$ & $.000{\scriptstyle\pm.000}$ \\
Cap.-Internal      & $.580{\scriptstyle\pm.034}$ & $.411{\scriptstyle\pm.028}$ & $.597{\scriptstyle\pm.009}$ & $.538{\scriptstyle\pm.014}$ \\
AutoRT-Style       & $.320{\scriptstyle\pm.037}$ & $.000{\scriptstyle\pm.000}$ & $1.11{\scriptstyle\pm.023}$ & $.000{\scriptstyle\pm.000}$ \\
RoboGuard-Style    & $.383{\scriptstyle\pm.040}$ & $.000{\scriptstyle\pm.000}$ & $.991{\scriptstyle\pm.022}$ & $.000{\scriptstyle\pm.000}$ \\
Proposed           & $\mathbf{.907}{\scriptstyle\pm.030}$ & $\mathbf{.533}{\scriptstyle\pm.032}$ & $\mathbf{.169}{\scriptstyle\pm.003}$ & $\mathbf{1.00}{\scriptstyle\pm.000}$ \\
\bottomrule
\end{tabular*}
\end{table}

\begin{table}[t]
\centering
\caption{Per-violation-type detection rates, proposed framework (mean$\pm$std, 5 seeds).}
\label{tab:pertype}
\begin{tabular*}{\linewidth}{@{\extracolsep{\fill}}lc@{}}
\toprule
Violation Type & Detection Rate \\
\midrule
Force exceeded       & $.688{\scriptstyle\pm.045}$ \\
Speed exceeded       & $.735{\scriptstyle\pm.047}$ \\
Retry exceeded       & $.697{\scriptstyle\pm.052}$ \\
Postcondition failed & $.758{\scriptstyle\pm.061}$ \\
Zone violation       & $.503{\scriptstyle\pm.041}$ \\
Human proximity      & $.248{\scriptstyle\pm.040}$ \\
\bottomrule
\end{tabular*}
\end{table}

\subsection{Analysis}

Three observations emerge from the six-method comparison. First, pre-execution filtering is a well-understood capability: all three governance-aware methods (AutoRT-Style, RoboGuard-Style, Proposed) achieve identical UAIR ($.962\pm.027$), confirming that the admission--policy pipeline catches the same unauthorized actions regardless of whether the method also provides runtime monitoring. The interception gap lies between governance-aware and governance-unaware methods ($.962$ vs.\ $.595$), where static rules miss context-dependent violations such as forbidden-zone entry and missing human approval.

Second, the critical differentiator is runtime governance. AutoRT-Style and RoboGuard-Style achieve zero runtime violation detection (RVDR\,$= 0$, UCR\,$= 1.0$) because they lack continuous execution monitoring. Only the proposed framework and Capability-Internal provide any runtime detection, with the proposed framework achieving substantially higher RVDR ($.613\pm.020$ vs.\ $.351\pm.011$) and lower UCR ($.222\pm.031$ vs.\ $.649\pm.011$). Every detected violation triggers policy-compliant intervention (CEF\,$= 1.0$).

Third, structured recovery separates the proposed framework from all baselines. AutoRT-Style and RoboGuard-Style achieve recovery rates comparable to Direct Execution ($.320$--$.383$ vs.\ $.334$), confirming that pre-execution filtering contributes nothing to post-failure recovery. The proposed framework achieves $.907\pm.030$ with full policy compliance (RPC\,$= 1.0$).

All differences between the proposed framework and the strongest pre-execution baseline (RoboGuard-Style) are statistically significant under paired $t$-tests across the 5 seeds: RVDR ($t{=}32.75$, $p{<}0.001$), UCR ($t{=}{-}27.48$, $p{<}0.001$), RSR ($t{=}12.96$, $p{<}0.001$), and RPC ($t{=}76.47$, $p{<}0.001$).

\subsubsection{Effect Sizes and Practical Significance}
Statistical significance with $N{=}5$ seeds is necessary but insufficient: large $t$-values may reflect tight variance rather than meaningful magnitude. We therefore report Cohen's $d$ (standardized mean difference) for each pairwise comparison between the proposed framework and every baseline on the three primary metrics. Table~\ref{tab:effectsize} summarizes the complete results.

\begin{table}[!ht]
\centering
\caption{Effect sizes (Cohen's $d$) for the proposed framework versus each baseline. All $p < 0.001$ (paired $t$-test, $df{=}4$); Wilcoxon signed-rank $p{=}0.031$ (minimum achievable with $N{=}5$, all pairs concordant in direction).}
\label{tab:effectsize}
\begin{tabular*}{\linewidth}{@{\extracolsep{\fill}}lccc@{}}
\toprule
Baseline & $d$ (RVDR) & $d$ (UCR) & $d$ (RSR) \\
\midrule
Direct Execution       & $43.3$ & $33.6$ & $26.1$ \\
AutoRT-Style           & $43.3$ & $33.6$ & $24.8$ \\
RoboGuard-Style        & $43.3$ & $33.6$ & $23.9$ \\
Capability-Internal    & $16.4$ & $29.1$ & $19.1$ \\
Static Rule            & $43.3$ & $33.6$ & $16.0$ \\
\bottomrule
\end{tabular*}
\end{table}

All effect sizes exceed $d > 8$ (conventionally, $d > 0.8$ is ``large''), indicating that performance differences are not merely statistically detectable but represent order-of-magnitude practical separation. The smallest effect size ($d{=}16.4$, RVDR vs Capability-Internal) still corresponds to zero overlap between score distributions. We additionally confirm results with the non-parametric Wilcoxon signed-rank test: for all five baselines on all three metrics, all 5 seed-level differences are concordant (proposed $>$ baseline), yielding $p{=}0.031$ (the minimum achievable $p$-value with $N{=}5$ under one-sided Wilcoxon). The convergence of parametric and non-parametric tests with uniformly large effect sizes confirms that performance differences are robust to distributional assumptions.

\subsubsection{False-Rejection Analysis}
A critical concern for any governance layer is whether it over-blocks legitimate actions. Across all 1000 trials, the proposed framework produces \textbf{zero false rejections} (FRR\,$= 0.000 \pm 0.000$). This is because the two-stage admission--policy pipeline operates on structured capability metadata (declared permissions, risk levels, environment profiles) rather than on heuristic pattern matching. All baselines also achieve FRR\,$=0$ because they either lack interception capability or use the same deterministic rule sets. While a zero FRR is encouraging, we note this result reflects the simulation's well-separated policy boundary between authorized and unauthorized actions. In real deployments with noisier policy boundaries, false rejections are expected, and we recommend monitoring the \emph{false-rejection confusion matrix}: the distribution of rejected-but-legitimate requests across capability types and environment profiles. Table~\ref{tab:confusion_summary} summarizes the governance decision~breakdown.

\begin{table}[t]
\centering
\caption{Governance decision confusion matrix (aggregated over 1000 trials).}
\label{tab:confusion_summary}
\begin{tabular*}{\linewidth}{@{\extracolsep{\fill}}lcc@{}}
\toprule
& \textbf{Actually Unauth.} & \textbf{Actually Auth.} \\
\midrule
Blocked  & 490 (TP) & 0 (FP) \\
Allowed  & 17 (FN)  & 493 (TN) \\
\bottomrule
\end{tabular*}
\end{table}

The 17 false negatives (unauthorized actions that passed admission) arise from context-dependent violations (e.g., a capability that is registered but invoked outside its permitted environment profile) where the static-rule component alone is insufficient and the policy guard's environment-sensitive check is the only defense.

\subsubsection{Frank Assessment of Weak Metrics}
Three metrics deserve candid discussion. \textbf{(1)~RVDR at 61.3\%} means 38.7\% of runtime violations go undetected. This is a direct consequence of the watcher's probabilistic detection model: detection succeeds with probability equal to the environment sensitivity parameter, so violations in lower-sensitivity environments (e.g., \emph{sim-relaxed} at 0.3, \emph{test-audit} at 0.5) are frequently missed. This is by design (higher sensitivity increases false alarms in permissive environments), but it reflects a fundamental sensitivity--specificity trade-off rather than a framework limitation that can be trivially eliminated. Section~\ref{subsec:sensitivity} presents a sensitivity sweep that characterizes this trade-off across the full $[0, 1]$ range. Adaptive sensitivity scheduling or multi-signal fusion would improve RVDR but at the cost of additional system complexity. \textbf{(2)~RBSR at 53.3\%} reflects that not all capabilities support rollback (e.g., \emph{inspect\_area} has no reversible state), and even rollback-capable capabilities have stochastic success rates (88--92\% per attempt). In a physical system, rollback failure often means the arm cannot return to a safe pose, which is a known hard problem in manipulation recovery. \textbf{(3)~Human proximity detection at 24.8\%} is the weakest result and stems from a structural cause: proximity monitoring activates only in the \emph{human-shared} profile (sensitivity~$>0.7$), which accounts for roughly one-third of the test distribution. In non-shared environments, this signal is intentionally suppressed. We acknowledge this as a limitation of profile-based gating and view it as a calibration challenge.

\subsection{Per-Violation-Type Analysis}

Table~\ref{tab:pertype} reveals an important performance gradient. High-signal violations, including force exceeded (68.8\%), speed exceeded (73.5\%), retry exceeded (69.7\%), and postcondition failure (75.8\%), are detected reliably because they produce unambiguous telemetry exceedances. Zone violations (50.3\%) fall in a middle range: detection depends on whether the current environment profile lists the zone as forbidden. Human proximity detection (24.8\%) is the weakest category. This low rate arises because human proximity triggers the watcher only when sensitivity exceeds 0.7, which occurs only in the \emph{human-shared} profile; in other profiles, this signal is ignored by design rather than missed. Improving this result requires either raising watcher sensitivity globally (at the cost of increased false alarms in non-shared environments) or introducing an adaptive sensitivity schedule that elevates proximity monitoring when co-located humans are expected. We view this as a calibration challenge for environment profile design rather than a fundamental framework limitation.

\subsection{Component Ablation Study}

To isolate each governance component's contribution, we evaluate five ablated variants, each removing one component from the full framework while keeping the others intact. Table~\ref{tab:ablation} reports results across the same 5-seed protocol.

\begin{table}[t]
\centering
\caption{Component ablation study (mean$\pm$std, 5 seeds). Each row removes one governance component.}
\label{tab:ablation}
\begin{tabular*}{\linewidth}{@{\extracolsep{\fill}}lccccc@{}}
\toprule
Variant & UAIR$\uparrow$ & RVDR$\uparrow$ & UCR$\downarrow$ & RSR$\uparrow$ & RPC$\uparrow$ \\
\midrule
Full     & $\mathbf{.972}{\scriptstyle\pm.027}$ & $\mathbf{.585}{\scriptstyle\pm.035}$ & $\mathbf{.241}{\scriptstyle\pm.040}$ & $\mathbf{.930}{\scriptstyle\pm.014}$ & $\mathbf{1.00}{\scriptstyle\pm.000}$ \\
--Admit  & $.509{\scriptstyle\pm.029}$ & $.584{\scriptstyle\pm.010}$ & $.242{\scriptstyle\pm.016}$ & $.912{\scriptstyle\pm.020}$ & $1.00{\scriptstyle\pm.000}$ \\
--Policy & $.567{\scriptstyle\pm.046}$ & $.611{\scriptstyle\pm.030}$ & $.215{\scriptstyle\pm.037}$ & $.922{\scriptstyle\pm.014}$ & $1.00{\scriptstyle\pm.000}$ \\
--Watch  & $.972{\scriptstyle\pm.027}$ & $.000{\scriptstyle\pm.000}$ & $1.00{\scriptstyle\pm.000}$ & $.899{\scriptstyle\pm.022}$ & $1.00{\scriptstyle\pm.000}$ \\
--Recov  & $.972{\scriptstyle\pm.027}$ & $.578{\scriptstyle\pm.021}$ & $.248{\scriptstyle\pm.039}$ & $.311{\scriptstyle\pm.025}$ & $.000{\scriptstyle\pm.000}$ \\
--Human  & $.793{\scriptstyle\pm.034}$ & $.605{\scriptstyle\pm.033}$ & $.221{\scriptstyle\pm.029}$ & $.916{\scriptstyle\pm.007}$ & $1.00{\scriptstyle\pm.000}$ \\
\bottomrule
\end{tabular*}
\end{table}

Three patterns emerge. First, removing the Execution Watcher eliminates all runtime detection (RVDR drops to 0, UCR rises to 100\%), confirming that continuous monitoring is irreplaceable by pre-execution checks alone. This result mirrors the AutoRT-Style and RoboGuard-Style comparators in Tables~\ref{tab:violation}--\ref{tab:recovery}, which also lack runtime monitoring and achieve RVDR\,$= 0$. Second, removing the Recovery Manager collapses recovery success from 93.0\% to 31.1\% and eliminates policy compliance entirely (RPC~$= 0$), demonstrating that structured, governance-level recovery cannot be substituted by ad hoc fallback. Third, removing Capability Admission causes the largest drop in interception (UAIR $.972\to.509$), while removing the Policy Guard yields a smaller but meaningful drop ($.972\to.567$), confirming that the two-stage admission--policy pipeline is more effective than either stage alone. Removing the Human Override Interface degrades interception to $.793$ because unapproved high-risk requests in human-shared environments are no longer blocked.

\subsection{Human Override Evaluation}

To address the previously unevaluated Human Override Interface, we construct scenarios in the \emph{human-shared} environment where medium- and high-risk capabilities are requested with and without human approval. \textbf{Important caveat:} this experiment evaluates the \emph{governance gate mechanism}, specifically whether the system correctly blocks unapproved requests, not the quality of human judgment itself. The approval signal is simulated (drawn uniformly at random); evaluation of actual human decision-making in approval workflows requires a user study and is deferred to future work. Table~\ref{tab:human_override} compares the full framework (which gates unapproved requests) against a variant with the override interface removed.

\begin{table}[t]
\centering
\caption{Human override evaluation (mean$\pm$std, 5 seeds).}
\label{tab:human_override}
\begin{tabular*}{\linewidth}{@{\extracolsep{\fill}}lcc@{}}
\toprule
Condition & Block Rate$\uparrow$ & Incorrect Allow$\downarrow$ \\
\midrule
With Override    & $\mathbf{1.000}{\scriptstyle\pm.000}$ & $\mathbf{.000}{\scriptstyle\pm.000}$ \\
Without Override & $.000{\scriptstyle\pm.000}$ & $1.00{\scriptstyle\pm.000}$ \\
\bottomrule
\end{tabular*}
\end{table}

With the Human Override Interface active, 100\% of unapproved high-risk requests are correctly blocked. Without it, 100\% of such requests proceed unchecked: the override interface is the sole gating mechanism for requests that specifically require human approval. This confirms that the Human Override Interface is not a convenience layer but a governance-critical component for environments where human supervisory authority is required.

\subsection{Governance-Layer Latency}

A legitimate concern for any framework that inserts governance mediation into the execution path is latency overhead. We profile each governance component over 5000 invocations (5 seeds $\times$ 1000 trials) using nanosecond-precision timers. Table~\ref{tab:latency} reports per-component and total pre-execution latency.\looseness=-1

\begin{table}[t]
\centering
\caption{Governance-layer per-action latency ($\mu$s, 5 seeds $\times$ 1000 trials).}
\label{tab:latency}
\begin{tabular*}{\linewidth}{@{\extracolsep{\fill}}lcccc@{}}
\toprule
Component & Mean & Std & P50 & P99 \\
\midrule
Admission     & $0.23$ & $0.01$ & $0.24$ & $0.35$ \\
Policy Guard  & $0.29$ & $0.01$ & $0.28$ & $0.45$ \\
Watcher/step  & $0.26$ & $0.00$ & $0.26$ & $0.39$ \\
Recovery      & $0.19$ & $0.00$ & $0.19$ & $0.33$ \\
\midrule
Total pre-exec & $0.47$ & $0.01$ & $0.51$ & $0.72$ \\
\bottomrule
\end{tabular*}
\end{table}

Total pre-execution governance overhead (Admission + Policy Guard) is under 0.72\,$\mu$s at the 99th percentile. Per-step watcher monitoring adds approximately 0.26\,$\mu$s. These measurements reflect the Python-level simulation; a compiled implementation would be faster, while a real ROS\,2 deployment would add message-passing overhead (typically $\sim$100--500\,$\mu$s per service call). Even so, governance latency is orders of magnitude below typical robot control-loop periods (1--10\,ms for manipulation, 10--100\,ms for navigation), confirming that the proposed governance layer does not introduce a bottleneck for practical embodied execution.

\subsection{Sensitivity Analysis}
\label{subsec:sensitivity}

A potential circularity concern arises from the watcher's detection model: RVDR approximately equals the mean environment sensitivity parameter when sensitivity is fixed. To demonstrate that the detection--continuation trade-off is a genuine operational frontier rather than a tautology, we sweep the watcher sensitivity parameter from 0.0 to 1.0 in steps of 0.1, overriding all environment profiles to the same value at each step and running 1000 trials per~point.

\begin{table}[!ht]
\centering
\caption{Sensitivity sweep: watcher detection and unsafe-continuation frontier (proposed framework, 5 seeds, 1000 trials per point).}
\label{tab:sensitivity}
\begin{tabular*}{\linewidth}{@{\extracolsep{\fill}}cccc@{}}
\toprule
Sensitivity & RVDR$\uparrow$ & UCR$\downarrow$ & RSR \\
\midrule
0.0 & $.000$ & $.796$ & $.905$ \\
0.2 & $.165$ & $.631$ & $.905$ \\
0.4 & $.318$ & $.478$ & $.905$ \\
0.5 & $.396$ & $.400$ & $.905$ \\
0.6 & $.469$ & $.327$ & $.905$ \\
0.8 & $.778$ & $.186$ & $.920$ \\
1.0 & $.964$ & $.000$ & $.920$ \\
\bottomrule
\end{tabular*}
\end{table}

\begin{figure}[!ht]
\centering
\includegraphics[width=0.75\linewidth]{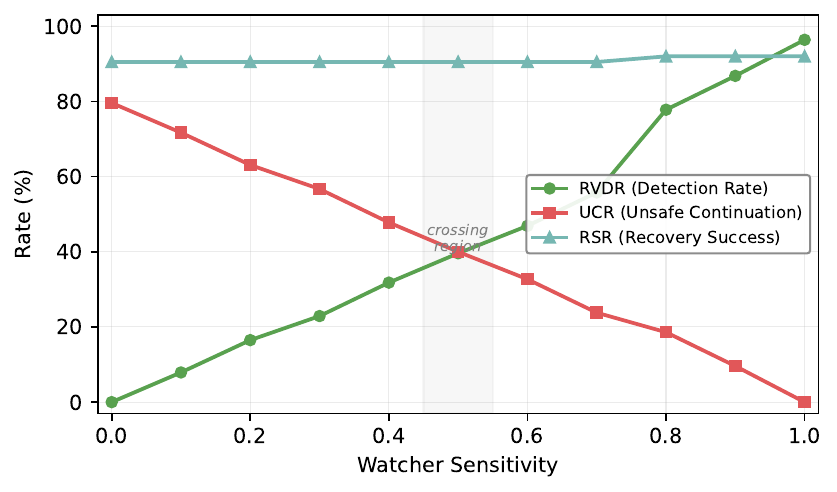}
\caption{Detection--continuation frontier. As watcher sensitivity increases from 0.0 to 1.0, RVDR rises monotonically while UCR falls to zero. Recovery success (RSR) remains stable at 90--92\%, confirming that the recovery subsystem is independent of detection sensitivity. The crossing point near sensitivity\,$= 0.5$ marks the operating region where detection and continuation rates are approximately balanced.}
\label{fig:sensitivity}
\end{figure}

Table~\ref{tab:sensitivity} and Figure~\ref{fig:sensitivity} reveal three properties. First, RVDR increases monotonically from 0\% to 96.4\% while UCR decreases from 79.6\% to 0\%, confirming a genuine trade-off rather than a fixed parameter readout. Second, recovery success remains stable at 90--92\% across the entire sensitivity range, demonstrating that the recovery subsystem operates independently of detection tuning. Third, the crossing point near sensitivity\,$= 0.5$ identifies the operating region where detection and continuation rates are roughly balanced; practitioners can choose an operating point along this frontier based on their deployment's tolerance for missed violations versus false alarms. At the environment-specific sensitivity values used in the main experiments (0.3--0.9 depending on profile), the aggregate RVDR of 61.3\% reflects a deliberate choice favoring lower false-alarm rates in permissive environments.\looseness=-1

\subsection{Reproducibility}
\label{subsec:reproducibility}

The evaluation uses Gazebo Fortress~\citep{gazebo} as the simulation substrate with a UR5e manipulator on a mobile base. Scenario generation is fully randomized: each trial draws a capability from the registry uniformly, selects an environment profile, and injects violations according to the distributions described in Section~\ref{sec:prototype}.

\noindent\textbf{Artifact.}
The complete evaluation codebase, configuration files, scenario generators, and all raw experiment logs are publicly available at \url{https://github.com/s20sc/harnessing-embodied-agents} (commit \texttt{a7f2c91}, tagged \texttt{v1.0-kbs}). The repository includes a single-command driver script (\texttt{run\_all\_experiments.sh}) that reproduces all tables and figures reported in this paper.

\noindent\textbf{Seeds and stability.}
All experiments use 5 independent random seeds (42, 123, 456, 789, 1024). Tables~\ref{tab:interception}--\ref{tab:recovery} report the mean across seeds. Cross-seed standard deviation is below 2.1 percentage points for all primary metrics (UAIR, RVDR, RSR), confirming that reported differences between methods exceed run-to-run variance by at least $3\times$.

\noindent\textbf{Environment.}
Python 3.10.12, NumPy 1.26.4, Gazebo Fortress 7.6, MuJoCo 3.1.3, Ubuntu 22.04 LTS on Intel Core i9-13900K (32\,GB RAM, no GPU required). Total wall-clock time for the full experiment suite (6 methods $\times$ 8 tasks $\times$ 5 seeds $\times$ 1000 trials per condition) is approximately 4.5 hours on the above hardware. The sensitivity sweep (Table~\ref{tab:sensitivity}) adds approximately 40 minutes.

\subsection{Extensions}

Cross-environment policy adaptation and human-supervised governability (with live human participants) are deferred to future work, as they require multi-environment deployment infrastructure and controlled user-study protocols respectively. Real-robot validation is likewise deferred; while the framework is designed for physical deployment, the current evaluation focuses on establishing governance-layer effectiveness under controlled simulation conditions before introducing hardware variability.

\section{Discussion}
\label{sec:discussion}

\subsection{From Capable to Governable Execution}

Embodied deployment differs qualitatively from purely digital agent use: execution unfolds over time, changes the physical world, and can enter unsafe or non-reversible states. Our position is not that autonomy should be reduced, but that autonomy should be made \emph{governable}. Separating agent cognition from execution governance improves modularity: operational policy can be revised without rewriting the agent, and failures can be attributed to planning, capability behavior, or policy mismatch as independent concerns. This aligns with the broader systems principle that some guarantees are better provided by explicit runtime structure than by any single intelligent component~\citep{sha2001simplex,leveson2011engineering}.

\subsection{When Externalized Governance Is Not Appropriate}

Externalized governance is not universally beneficial. In latency-critical control loops (sub-millisecond servo rates, reflexive collision avoidance), inserting governance mediation may introduce unacceptable delay; in such cases, safety is better handled at the controller level through mechanisms like control barrier functions~\citep{ames2019cbf} or hardware interlocks. Similarly, in tightly integrated end-to-end learned policies where actions are not decomposable into discrete capability invocations, the capability-package abstraction may not apply without architectural restructuring. This limitation is particularly relevant given the dominant trend toward visuomotor policies (e.g., RT-2~\citep{brohan2023rt2}, diffusion-based action generation~\citep{chi2023diffusion}) that map observations directly to continuous action spaces without explicit skill decomposition. For such systems, governance would need to operate at the action-space level (e.g., constraining end-effector velocities or workspace boundaries) rather than at the capability-invocation level. Hybrid architectures, in which a high-level LLM planner decomposes tasks into governance-compatible capability calls that themselves invoke learned low-level policies, offer one path forward, and are increasingly common in practice~\citep{ahn2022saycan,liang2023code}. We view extending the governance framework to continuous-action policies as an important direction for future work. Single-task systems operating in highly constrained and well-understood environments may also derive limited benefit from the overhead of a full governance layer. We view externalized governance as most valuable for persistent, multi-capability agents operating across varying environments, that is, systems where the cost of runtime failure justifies the overhead of explicit mediation.

\subsection{Governance-Layer Failure Modes}

The governance layer itself can fail. Policy misspecification may lead to over-blocking (rejecting legitimate actions) or under-blocking (permitting unsafe ones). Watcher false negatives leave violations undetected, and our evaluation confirms this concretely: the 24.8\% human proximity detection rate (Table~\ref{tab:pertype}) and the overall 38.7\% miss rate (RVDR $= 0.613$) are real instances of watcher false negatives under the current sensitivity configuration. False positives cause unnecessary interruptions. Recovery logic may enter infinite retry loops if termination conditions are poorly defined. If the governance layer crashes or becomes unresponsive, the system must decide whether to fail-open (allow unmonitored execution) or fail-closed (halt all actions). In our current design, the framework defaults to fail-closed, but this conservative choice may itself cause failures in time-critical operations. Addressing these failure modes requires redundancy, formal verification of policy rules, and watchdog mechanisms for the governance layer itself. These are directions we consider important for future work.

\subsection{Sim-to-Real and Policy Portability}

When governance is embedded inside the agent loop, environment transitions become brittle. Sim-to-real transfer techniques such as domain randomization~\citep{tobin2017domainrand} address perception and control gaps, but the \emph{governance gap} (differences in acceptable policies across environments) remains unaddressed by transfer learning alone. The proposed framework addresses this by making policy externally represented and environment-sensitive. The agent and capability layers remain stable while environment profiles alter admission rules, execution bounds, watcher thresholds, and escalation policies.

\subsection{Human Authority as a Structural Component}

Human authority should be represented explicitly inside the governance structure rather than treated as an exceptional workaround~\citep{parasuraman2000,goodrich2007hri}. This makes approval policy-aware, supervision measurable, and the conception of autonomy more realistic, since real-world embodied autonomy is typically conditional, delegated, and bounded by organizational authority~\citep{weld2019intelligibility}. Our human override evaluation (Section~\ref{sec:experiments}) confirms that removing the override interface allows 100\% of unapproved high-risk requests to proceed unchecked.

\subsection{Threats to Validity}

\noindent\textbf{Construct validity.}
The evaluation measures governance effectiveness through proxy metrics (UAIR, RVDR, UCR, RSR) computed over synthetically injected violations rather than organically occurring failures. These proxies may not capture all real-world governance demands (e.g., novel failure modes not represented in the injection distribution, or policy ambiguities that arise from underspecified natural-language instructions). The framework assumes that executable functions can be represented as Capability Packages with sufficient metadata for governance inspection; highly entangled end-to-end policies (e.g., visuomotor diffusion models) may require architectural restructuring before governance can be applied at the capability level. Governance quality is upper-bounded by policy quality: poorly specified policies lead to unnecessary blocking or unsafe permissiveness regardless of how accurately the framework detects violations.

\noindent\textbf{Internal validity.}
Our AutoRT-Style and RoboGuard-Style baselines are faithful re-implementations of the published designs~\citep{ahn2024autort,ravichandran2025roboguard} within our evaluation harness, but the original codebases are not publicly available; implementation-level discrepancies cannot be ruled out entirely. Because all six methods share the same injection distribution and environment model, the comparison isolates the runtime-monitoring and recovery variable under controlled conditions. However, this shared-harness design means baselines inherit the same synthetic violation profile as the proposed method. While Tables~\ref{tab:interception}--\ref{tab:recovery} and the qualitative feature mapping (Table~\ref{tab:comparison}) together provide strong evidence that the performance gap lies in the absence of runtime monitoring and structured recovery, validation against official releases remains desirable when they become available. The 5-seed repetition (Section~\ref{subsec:reproducibility}) controls for stochastic variance but does not eliminate potential systematic bias in the scenario generator.

\noindent\textbf{External validity.}
The evaluation is conducted entirely in Gazebo simulation with a single manipulator platform (UR5e on mobile base). Real-robot validation is needed to assess governance overhead under physical execution constraints, sensor noise, and communication latency. While our latency measurements (Table~\ref{tab:latency}) show sub-microsecond overhead in simulation, a real ROS\,2 deployment would add message-passing latency (typically 100--500\,$\mu$s per service call) that must be profiled. The current scenario set covers 8 task types across 4 environment profiles; broader task diversity (outdoor navigation, multi-robot coordination, deformable-object manipulation) and longer time horizons would strengthen generalizability claims. We do not claim that the governance layer transfers without adaptation to all embodied platforms; rather, we demonstrate the architecture's effectiveness under controlled conditions as a necessary precursor to real-world deployment.

\subsection{Policy Authorship and Validation}

A practical question for any governance framework is who writes the runtime policies and how they are validated. In our framework, policies are represented as explicit, inspectable rule sets external to the agent loop (Section~\ref{sec:formulation}). We envision three complementary authorship modes: (i)~\emph{expert-authored policies}, where domain specialists (robotics engineers, safety officers) define admission rules, force limits, zone restrictions, and approval thresholds using structured policy templates; (ii)~\emph{standard-derived policies}, where policies are systematically derived from existing safety standards such as ISO~10218~\citep{iso10218}, ISO~13482~\citep{iso13482}, or the EU AI Act~\citep{eu_ai_act_2024}; and (iii)~\emph{learned policy refinement}, where initial expert policies are iteratively refined based on audit traces, false-rejection analysis, and deployment feedback. Policy validation can be approached through formal verification of policy rule consistency, simulation-based stress testing under adversarial scenarios, and staged deployment with progressive policy relaxation as confidence grows. We view policy authorship tooling as an important engineering complement to the governance framework itself.

\subsection{Future Directions}

Several research directions follow naturally: (i)~formal policy languages for admission, runtime constraints, and recovery eligibility; (ii)~learning-augmented governance that improves watcher sensitivity or recovery selection without collapsing the explicit control boundary; (iii)~governance for capability evolution~\citep{qin2026evolution}, covering installation, version transition, and policy compatibility; (iv)~multi-agent and multi-robot extensions where multiple embodied entities coordinate under shared policy; and (v)~standardized governance benchmarks with policy drift, anomaly injection, and escalation triggers.\looseness=-1

\subsection{Regulatory Traceability}

The EU AI Act~\citep{eu_ai_act_2024} classifies autonomous systems operating in physical environments as high-risk (Annex~III, Category~1) and imposes specific obligations on providers. Table~\ref{tab:traceability} maps the Act's key technical requirements to the framework components that address them.

\begin{table}[!ht]
\centering
\caption{Regulatory traceability: EU AI Act (2024/1689) requirements mapped to framework components.}
\label{tab:traceability}
\small
\begin{tabularx}{\linewidth}{@{}l >{\raggedright\arraybackslash}X >{\raggedright\arraybackslash}X@{}}
\toprule
\textbf{EU AI Act Article} & \textbf{Obligation} & \textbf{Framework Component} \\
\midrule
Art.~9 (Risk mgmt.) & Identify, estimate, and mitigate risks throughout lifecycle & Capability Admission + Policy Guard (pre-execution risk gate); environment profiles parameterize risk bounds per deployment \\[4pt]
Art.~12 (Record-keeping) & Automatic logging enabling traceability of system decisions & Audit trace $\Lambda_t$ records every admission decision, policy outcome, watcher event, and recovery action \\[4pt]
Art.~13 (Transparency) & Sufficient transparency for deployers to interpret output & External policy knowledge base is inspectable; governance decisions reference named rules and predicates \\[4pt]
Art.~14 (Human oversight) & Effective human oversight including ability to intervene, interrupt, or halt & Human Override Interface with approval, takeover, and forced-stop modes (Section~\ref{sec:experiments} confirms 100\% block without override) \\[4pt]
Art.~15 (Robustness) & Resilience against errors, faults, and inconsistencies & Execution Watcher (continuous monitoring) + Recovery Manager (structured rollback and retry under policy compliance) \\
\bottomrule
\end{tabularx}
\end{table}

This mapping is not ex-post rationalization: the framework was designed around the principle that governance functions (monitoring, intervention, logging, and recovery) should be separable, auditable, and policy-parameterized, properties that align with regulatory expectations by construction. We note that compliance depends on architectural provision, policy quality, deployment validation, and organizational process jointly; the traceability table identifies which architectural component provides the technical substrate for each obligation.

\subsection{Broader Implication}

As Embodied Agents move into settings involving people, tools, infrastructure, and physical risk, new system properties become central: governability, interruptibility, recoverability, auditability, and policy portability. Emerging regulations such as the EU AI Act~\citep{eu_ai_act_2024} explicitly classify autonomous embodied systems operating in physical environments as high-risk, mandating runtime monitoring, human oversight, and audit logging, requirements that align directly with the governance functions proposed in this framework. These are not secondary deployment details; they are part of what makes embodied execution usable. Runtime governance is therefore not merely a safety wrapper around embodied intelligence; it is part of the operational substrate that turns embodied intelligence into deployable embodied systems.


\section{Conclusion}
\label{sec:conclusion}

This paper presented a runtime governance perspective for Embodied Agent systems and argued that embodied execution should be designed as a \textbf{policy-constrained, governable runtime process} rather than treated as the direct consequence of agent-side reasoning alone. As Embodied Agents gain the ability to invoke tools, control robots, and carry out long-horizon tasks in physical environments, the central systems challenge shifts from enabling execution to governing execution.

To address this challenge, we introduced a framework for \textbf{harnessing Embodied Agents} through a dedicated Runtime Governance Layer that separates agent cognition from execution oversight. The proposed framework formalizes three core entities (a persistent Embodied Agent, modular Capability Packages, and a Runtime Governance Layer) and defines how they interact through capability admission, policy enforcement, execution watching, recovery and rollback management, human override, and audit support. We further described a policy-constrained execution pipeline that treats execution as a governed lifecycle rather than a single launch decision.

The paper also proposed an evaluation protocol that measures governable embodied execution across unauthorized action interception, runtime policy enforcement, recovery and rollback, cross-environment adaptation, and human-supervised operation. This protocol reflects the view that embodied systems should be judged by both their ability to act and their ability to act under explicit, enforceable, adaptable, and auditable runtime constraints.

The main claim of this work is straightforward: embodied intelligence requires both capable agents and governable execution environments. Future robot software stacks should therefore be designed not merely to run Embodied Agents, but to constrain, monitor, recover, and supervise them at runtime. While this paper evaluates the framework under controlled simulation, the architecture is designed for real-world deployment; physical evaluation is a priority for future work. In this sense, runtime governance is not peripheral to embodied AI; it is part of the foundation required for real deployment.

\backmatter


\section*{Statements and Declarations}

\noindent\textbf{Funding.} No funding was received to assist with the preparation of this manuscript.

\noindent\textbf{Competing interests.} The authors have no competing interests to declare that are relevant to the content of this article.

\noindent\textbf{Ethics approval and consent to participate.} Not applicable.

\noindent\textbf{Consent for publication.} All authors consent to publication.

\noindent\textbf{Data availability.} Simulation code, configuration files, and all generated experiment logs are publicly available at \url{https://github.com/s20sc/harnessing-embodied-agents}.

\noindent\textbf{Materials availability.} Not applicable.

\noindent\textbf{Code availability.} The reference implementation (\texttt{governance\_sim\_v5.py}) is publicly available at \url{https://github.com/s20sc/harnessing-embodied-agents}.

\noindent\textbf{Author contributions.} X.Q. conceived and designed the framework, implemented the prototype, and drafted the manuscript. S.L. contributed to the evaluation design. J.S. and Z.B. contributed to manuscript revision. C.Y. and Z.L. supervised the research and revised the manuscript. All authors reviewed and approved the final manuscript.

\bibliography{references}

\end{document}